\documentclass{amsart}

\RequirePackage{amsthm,amsmath,amsfonts,amssymb,comment}
\RequirePackage[numbers,sort&compress]{natbib}
\RequirePackage[colorlinks,citecolor=blue,urlcolor=blue]{hyperref}
\RequirePackage{graphicx}

\usepackage{amsthm}% http://ctan.org/pkg/amsthm
\usepackage{hyperref}% http://ctan.org/pkg/hyperref
\hypersetup
{colorlinks = true,
linkcolor = red, 
anchorcolor = red, 
citecolor = blue, 
filecolor = blue,
urlcolor = blue}
\usepackage{hypcap}
\usepackage{cleveref}% http://ctan.org/pkg/cleveref
\usepackage{lipsum}% http://ctan.org/pkg/lipsum
\usepackage{enumerate}
\usepackage{enumitem}
\usepackage{amsmath,amsfonts,amssymb,mathrsfs, amscd,amsthm,amsbsy,amsxtra,bbm,bm, epsf,calc,comment, xcolor}
\usepackage[toc,page]{appendix}
\usepackage{color}
\usepackage{datetime}
\usepackage{latexsym}
\usepackage[english]{babel}
\usepackage{graphicx}
\usepackage{epsfig}
\usepackage{dsfont}
\usepackage{tikz}
\usetikzlibrary{arrows.meta,patterns,positioning}

\usepackage{soul}
\usepackage{multirow}

\usepackage{algorithmic} %For computer algorithm code in LATEX
\usepackage{algorithm} %For algorithm box around the algorithmic
\usepackage{dsfont}

\theoremstyle{plain}

\numberwithin{equation}{section}\newtheorem{theorem}{Theorem}[section]
\newtheorem{lemma}[theorem]{Lemma}
\newtheorem{definition}[theorem]{Definition}
\newtheorem{proposition}[theorem]{Proposition}
\newtheorem{corollary}[theorem]{Corollary}
\newtheorem{example}[theorem]{Example}

\newtheorem{assumption}[theorem]{Assumption}

\theoremstyle{remark}
\newtheorem{remark}[theorem]{Remark}

\newcommand{\norm}[1]{\lVert #1 \rVert}

\definecolor{mygreen}{rgb}{0.1,0.75,0.2}

\DeclareMathOperator*{\argmin}{argmin}

\DeclareSymbolFont{bbold}{U}{bbold}{m}{n}
\DeclareSymbolFontAlphabet{\mathbbold}{bbold}

\newcommand{\spt}{\textup{spt}}

\newcommand{\X}{\mathcal{X}}

\newcommand{\R}{\mathbb{R}}

\usepackage{kotex}

\title[Statistical Inference for Adversarial Training]{Statistical Inference for Adversarial Training: Central Limit Theorems via Optimal Transport}

\begin{document}

\author{Jakwang Kim}
\address{School of Data Science, The Chinese University of Hongkong, Shenzhen, Dao Yuan Building, 2001 Longxiang Boulevard, Longgang District, Shenzhen, Guangdong, China.}
\email{jakwangkim@cuhk.edu.cn}
%%%%%%%%%%%%%%%%%%%%%%%
\author{Dohyun Kwon}
\address{School of Mathematics \& Computing (Computational Science \& Engineering), Yonsei University, Seodaemun-gu, Seoul 03722, Republic of Korea.}
\address{Center for AI and Natural Sciences, Korea Institute for Advanced Study, Dongdaemun-gu, Seoul 02455, Republic of Korea.}
\email{dohyunkwon@yonsei.ac.kr}
%%%%%%%%%%%%%%%%%%%%%%%

\date{\today}
\thanks{
JK is supported by CUHK-SZ start-up UDF03004229. DK is partially supported by the National Research Foundation of Korea (NRF) grant funded by the Korea government (MSIT) (No. RS-2023-00252516, No. RS-2024-00408003,  No. RS-2026-25488663, and No. RS-2026-25613008), the POSCO Science Fellowship of POSCO TJ Park Foundation, and the Korea Institute for Advanced Study. This research was supported by the Yonsei University Research Fund of 2026-22-0251.
}

\begin{abstract}
The purpose of this paper is to rigorously quantify the statistical and learning-theoretic properties of adversarial training models for classification. Equivalently, we establish the statistical properties of empirical optimal partial transport. Precisely, first we provide two types of central limit theorems (CLT): CLT centered at the expected empirical value, and CLT centered at the population one with smoothing. These results are based on the uniqueness of optimal potential for various equivalent optimal transport formulations, and the empirical process theory argument. For the binary setting, we indeed prove the uniqueness of optimal potential by leveraging the connection between optimal partial transport and the derived multi-marginal optimal transport formula. As byproducts, we also obtain the stability of a saddle point of the adversarial training model, and the sample complexity and concentration probability of the generalization error.
\end{abstract}

\maketitle
\tableofcontents

\section{Introduction}
Over the last decade, deep learning-based machines, particularly parametrized by neural networks, have demonstrated remarkable performance across a wide range of learning tasks; see, for instance, \cite{lecun2015deep, krizhevsky2012imagenet} and the references therein. Despite their empirical success, trained neural networks can be highly sensitive to small perturbations of their inputs, as first observed in \cite{SzegedyZSBEGF13}. Such perturbations, known as \emph{adversarial attacks}, are commonly used to assess the worst-case behavior of learning models. Subsequent work has demonstrated that adversarial examples are not merely artificial phenomena but can arise in a variety of real-world settings, including sticker-based attacks on stop-sign recognition systems \cite{eykholt2018robust}, adversarial eyeglasses designed to evade face-recognition systems \cite{sharif2016accessorize}, and inaudible voice commands in the DolphinAttack \cite{10.1145/3133956.3134052}. Moreover, adversarial examples arise not only in image classification but also in other areas of artificial intelligence, including adversarial attacks on robotic vision \cite{jones2025adversarialattacksroboticvision} and advanced jailbreak strategies for large language models (LLMs) \cite{zou2023universal, 10.1145/3605764.3623985, jiang2024artprompt, pmlr-v267-sabbaghi25a}. We refer the reader to the surveys \cite{QIAN2022108889, zhao2024adversarialtrainingsurvey} for further references.

To mitigate this instability, \citet{goodfellow2014explaining} introduced the \emph{adversarial training model} to enhance the robustness of classifiers against adversarial perturbations. Since adversarial training is explicitly designed to counteract the instability of standard classification procedures, one naturally expects the resulting classifiers to exhibit improved robustness. This expectation is supported by extensive empirical evidence \cite{MadryMSTV18, li2025adversarial, BAN20243535}, as well as by a growing body of work devoted to its theoretical understanding \cite{bose2020adversarial, Meunier2021MixedNE, pmlr-v206-balcan23a, dobriban2023provable, javanmard2023adversarial, pmlr-v125-javanmard20a, hassani2024curse, li2025adversarial}. Nevertheless, despite substantial progress over the past several years, a quantitative and comprehensive understanding of the statistical behavior of adversarially trained classifiers remains incomplete.

The purpose of this paper is to address this gap by studying a representative model of adversarial training. More precisely, we investigate the stability of robust classifiers and optimal adversarial attacks, the generalization error and its associated sample complexity, and the limiting distributions of adversarial training risks. Our analysis is carried out within a \emph{distributionally robust optimization} framework, which provides a general methodology for worst-case optimization and has been extensively studied in finance, operations research, and statistical learning; see, for example, \cite{delage2010distributionally, mohajerin2018data, kuhn2019wasserstein, MR3959085, MR4015639, blanchet2023unifying, wang2024outlier, Kuhn_Shafiee_Wiesemann_2025}.

A particularly useful formulation for our purposes is the \emph{distributional-perturbing adversarial model}. In the binary classification setting, a series of works has established connections between optimal transport and distributional-perturbing adversarial models \cite{Bhagoji2019LowerBO, Pydi2021AdversarialRV, pydi2021the, dai2023characterizing, Trillos2020AdversarialCN, NEURIPS2023_81858558, JMLR:v25:23-0456, frank2025adversarialsurrogateriskbounds}. In a series of papers \cite{jakwang_MOT, jakwang_2024existence, trillos2024optimal}, these connections were extended to the multiclass setting. In particular, the distributional-perturbing adversarial model admits a geometric interpretation through a \emph{generalized barycenter problem}, extending the classical Wasserstein barycenter problem \cite{Carlier2010MatchingFT, Agueh_Carlier2011}, and can further be connected to a multimarginal optimal transport (MOT) problem. This connection yields an exact optimal transport formula for the adversarial risk, together with an approximation algorithm. More recently, this framework has been extended to adversarial models with nonlinear loss functions \cite{trillos2025lowerboundsadversarialrobustness}.

These optimal transport formulations provide the starting point for our statistical analysis. A substantial body of work has investigated the asymptotic behavior of empirical optimal transport costs. While the limiting distributions of empirical optimal transport costs on finite spaces have been extensively studied \cite{Sommerfeld_Munk2018, Klatt_etc2022, asymptotic_LP2023, CLT_semidiscrete_OT2024}, the corresponding theory in continuous settings remains considerably less developed. One of the main difficulties arises from the size and complexity of the dual function class. At a high level, classical empirical process theory is generally not directly applicable to the dual class of optimal transport, since its metric entropy typically grows too rapidly except in low-dimensional settings. This difficulty stems from the limited regularity of optimal transport potentials under general assumptions.

Several approaches have been proposed to overcome this obstacle. One approach is to center the empirical transport cost by its expectation, for which central limit theorems can be established under suitable conditions \cite{CLT_OT2019, CLT_general_transport2024, Hundrieser_etc2024}. Other approaches modify the transport problem in a manner that reduces the complexity of the associated function class. These include entropic regularization \cite{CLT_EOT_finite, CLT_general_cost_OT, gonzalezsanz2024weaklimitsentropyregularized, gonzalezsanz2023weaklimitsempiricalentropic, lower_complexity_EOT, limit_distribution_EOT, samplecomplexity_EOT} and sliced Wasserstein distances \cite{Okano_Imaizumi2024, Goldfeld_Kato_Rioux_Sadhu2024, rodriguezvitores2025improvedcentrallimittheorem}. In particular, recent work on smoothed Wasserstein distances \cite{limit_pWasserstein2024, Goldfeld_Kato_Rioux_Sadhu2024} provides another route to limiting distributions by combining regularization of the underlying function class with functional delta methods.

A key ingredient in these asymptotic theories is the uniqueness of the optimal potential up to additive constants. Such uniqueness is essential for identifying the derivative of the transport functional and, consequently, the variance of its limiting distribution; see, for example, \cite{CLT_OT2019, Goldfeld_Kato_Rioux_Sadhu2024, CLT_general_transport2024}. The uniqueness of optimal potentials depends crucially on the properties of both the cost function and the underlying measures, including the regularity of the cost and density functions, as well as topological properties of the supports. Sufficient conditions ensuring uniqueness up to additive constants have been studied in \cite{CLT_OT2019, CLT_general_transport2024}; these include \Cref{assumption: strict convexity of cost function}, absolute continuity of the underlying measures, and connectedness of their supports with negligible boundaries.

The adversarial training problem considered here, however, does not fall directly within the scope of these existing results. There are two fundamental differences from standard optimal transport. First, the class-dependent measures $\mu_i$'s generally have different total masses, whereas standard optimal transport is formulated between measures of equal mass. Second, the multiclass adversarial training problem involves multiple marginals and therefore requires a genuinely multimarginal analysis. Consequently, existing CLTs for standard optimal transport cannot be applied directly to obtain limiting distributions for adversarial training risks.

There is an additional difficulty concerning uniqueness.  In \Cref{sec: optimal partial transport} we show that the distributional adversarial model \eqref{def: distributional model} is equivalent to the \emph{optimal partial transport} problem \eqref{def: partial transport}. Although uniqueness of the optimal potential up to additive constants is essential in existing limit-distribution results such as \cite{CLT_OT2019, Goldfeld_Kato_Rioux_Sadhu2024, CLT_general_transport2024}, the required uniqueness in optimal partial transport does not follow directly from the existing optimal transport literature. In particular, apart from the quadratic-cost setting \cite{CM2010Annalsmath}, general uniqueness results for optimal potentials in optimal partial transport are not available to the best of the authors' knowledge. Establishing the appropriate uniqueness result for the optimal potential is therefore an essential step in our analysis.

The main contribution of this work is to establish central limit theorems (CLTs) for the adversarial training risk, or equivalently, for the associated optimal partial transport problem. More precisely, we prove two types of CLTs: (1) a CLT centered at the expectation of the empirical adversarial risk, and (2) a CLT for the smoothed empirical adversarial risk centered at its population counterpart. Our analysis combines the optimal transport structure of adversarial training with tools from asymptotic statistics and empirical process theory.

The first CLT is obtained through an Efron--Stein argument. In the general multiclass setting, or equivalently the associated multimarginal optimal transport setting, we first establish stability and pointwise convergence of optimal potentials. A key observation is an unexpected Lipschitz regularity of these potentials, which follows from their natural boundedness together with the regularity of the cost function. Under the uniqueness of the optimal potential, these stability properties allow us to derive a CLT centered at the expected empirical adversarial risk.

A distinctive feature of the limiting distribution is the presence of nonvanishing cross-covariance terms,
\[
    -\mu_i(g_i)\mu_j(g_j),
\]
which arise from the dependence structure induced by the measure $\mu$ defined in \eqref{eq:mui0}. These terms reflect the intrinsic dependence among samples associated with different classes, a feature absent from the standard empirical optimal transport framework, where empirical marginals are typically sampled independently. Consequently, existing CLTs for empirical optimal transport do not directly apply to our setting.

We next establish a population-centered CLT through smoothing. Smoothing enables the use of empirical process theory, but the standard theory is not directly applicable because the associated potentials are vector-valued. We overcome this difficulty by lifting a finite-dimensional vector-valued potential to a real-valued function on a natural extended space. Together with the almost sure $L^2$-convergence of optimal potentials obtained from their stability, this lifting allows us to establish weak convergence of the corresponding empirical process to a Gaussian process indexed by the class of smoothed potentials, again exhibiting nonvanishing cross-covariance terms. This Gaussian limit then yields the CLT for the smoothed empirical adversarial risk centered at its population counterpart.

To establish the uniqueness required for both CLTs, we prove a new uniqueness result for optimal potentials in the binary setting. The proof exploits the equivalence among the adversarial training problem, its multimarginal optimal transport formulation and corresponding dual problem, and optimal partial transport. By moving between the multimarginal and partial transport formulations, we show that, when the input measures are absolutely continuous and have connected supports, the optimal potential is genuinely unique in the non-degenerate case (genuine optimal partial transport case) and unique up to additive constants in the degenerate case (standard optimal transport case). To the best of our knowledge, this is the first uniqueness result for optimal potentials in optimal partial transport beyond the quadratic-cost setting.

Our analysis also yields several statistical and learning-theoretic consequences. In particular, we establish the stability of saddle points consisting of robust classifiers and optimal adversarial attacks, including convergence of their empirical counterparts to population-level solutions. We further prove uniqueness of the robust classifier in the binary setting and uniqueness of the optimal adversarial attack in the general multiclass setting. Finally, we derive quantitative bounds on the generalization error of the adversarial training model, which in turn yield sample-complexity guarantees.

Taken together, these results provide a statistical theory for adversarial training that connects the geometry of optimal and partial transport with asymptotic statistics. The optimal transport formulations provide the stability and uniqueness properties needed to identify the limiting distributions, while the Efron--Stein argument and empirical process theory yield the two complementary CLTs.

\subsection{Our contributions}

We now summarize the main contributions of this paper.

\begin{itemize}
\item \textbf{Central limit theorem with empirical centering (\Cref{thm: empirical centering}, \Cref{Cor: empirical centering}).}
Using the equivalence between the adversarial training problem and its multimarginal optimal transport formulation, we establish a central limit theorem for the empirical adversarial training risk centered by its expectation. A key ingredient is the Efron-Stein inequality with a suitable lifting argument.

\item \textbf{Central limit theorem for the smoothed adversarial risk (\Cref{thm: smooth limit}, \Cref{Cor: smoothed centering}).}
We establish a central limit theorem for the smoothed empirical adversarial training risk centered by the corresponding population risk. The proof combines empirical process theory with a limiting Gaussian process having the appropriate covariance structure and the stability of optimal potential. This provides a population-centered asymptotic distribution complementary to the first CLT.

\item \textbf{Uniqueness (\Cref{thm: uniqueness}).}
Not only do we prove the central limit theorem assuming the uniqueness of the optimal potential, but we also show it (up to additive constants for the degenerate case) for the binary setting under general conditions. Its proof exploits the equivalence between the adversarial training problem and optimal partial transport and, in particular, addresses a uniqueness issue that is not covered by existing results in the optimal transport community.

Furthermore, we also prove the uniqueness of an optimal adversarial attack. Note that this result holds for the general multiclass case.

\item \textbf{Convergence of robust classifiers and adversarial attacks (\Cref{thm: convergence}).}
As the first byproduct of the analysis regarding the central limit theorems, we establish the stability of robust classifiers and optimal adversarial attacks with respect to perturbations of the input data. As a consequence, we prove the convergence of empirical robust classifiers and empirical optimal adversarial attacks to their population-level counterparts.

\item \textbf{Generalization error and sample complexity (\Cref{thm: generalization error}).}
As the second byproduct, we derive bounds on the generalization error of the adversarial training model. These bounds, in turn, combined with the recent statistical optimal transport results, yield quantitative sample-complexity guarantees for adversarial training.
\end{itemize}

\subsection{Problem setting}
\label{sec:proset}
Let $\mathcal{X}:=\mathbb{R}^p$ and $\mathcal{Y}:=\{1, \dots, K\}$ be the domain of feature vectors and the set of classes, with $K \geq 2$, respectively. Also, let $\mu \in \mathcal{P}(\mathcal{X} \times \mathcal{Y})$ be the ground-truth distribution where $\mathcal{P}(\mathcal{S})$ denotes the set of probability measures over $\mathcal{S}$ equipped with the weak topology.  More precisely, for any Borel measurable set $A \subseteq \mathcal{X}$ and $i \in \mathcal{Y}$,
\begin{align}\label{eq:mui0}
    \mu(X \in A, Y=k) = \int_{A}  d\mu_i(x).
\end{align}
Each $\mu_i$ can be regarded as an (unnormalized) conditional probability of $X$ given $Y=k$. We denote the marginals of $\mu$ by
\begin{align}
\label{eq:marg}
    d\mu_X := \sum_{i\in\mathcal Y} d\mu_i,
    \quad
    \mu_Y(i) := |\mu_i|
    = \int_{\mathbb R^p} d\mu_i(x),
    \quad i\in\mathcal Y.
\end{align}
We assume that $K$ is fixed and $\mu_Y(k)> 0$ for all $i \in \mathcal{Y}$.

Let $\mu^n$ and $\mu_i^n$ denote empirical distributions of $n$ i.i.d. samples:
\begin{align}
\label{eq:empdis}
    \mu^n = \frac{1}{n} \sum_{k=1}^n \delta_{(X_k, Y_k)}, \quad \mu_i^n := \frac{1}{n} \sum_{k=1}^n \delta_{(X_k, Y_k)} \delta_{Y_k=i}.
\end{align}
$\mu^n$ weakly converges to $\mu$ (hence so as each $\mu_i^n$ to $\mu_i$) as $n \to \infty$.

Under the assumption that data points are distributed according to the ground-truth distribution $\mu$, a classification problem can be mathematically modeled as
\begin{equation}\label{eq:risk0}
    \inf_{\theta \in \Theta} R(f_\theta, \mu), \quad R(f_\theta, \mu):=\mathbb{E}_{(X,Y) \sim \mu} [\ell (f_\theta(X), Y)) ]
\end{equation} 
for a given loss function $\ell: \mathcal{Y} \times \mathcal{Y} \rightarrow \mathbb{R}$ where $\theta$ is a parameter. In practice, people extend a notion of classifier to a probabilistic classifying rule over $\mathcal{Y}$; i.e., $f_\theta$ is a Borel measurable map from $\X$ into $\Delta_{\mathcal{Y}}$, where
\[
    \Delta_{\mathcal{Y}} := \left\{ (u_i)_{i \in \mathcal{Y}} : 0 \leq u_i \leq 1, \, \sum_{i \in \mathcal{Y}} u_i = 1 \right\}.
\]

We are interested in the adversarial training model, whose formal description is as follows. Fix a positive constant $\varepsilon>0$, called the \emph{adversarial budget}. For each feature vector $x$, the adversary chooses $x'$ around $x$ within a distance $\varepsilon$ such that $x'$ maximizes a loss given a classifier $f_\theta$. Mathematically, it is written as the following min-max game:
\begin{equation}\label{def: adversarial attack}
    \inf_{\theta \in \Theta} \mathbb{E}_{(X,Y) \sim \mu}  \left[ \sup_{X': d(X,X') \leq \varepsilon}\ell (f_\theta(X'), Y)) \right].
\end{equation}
Although its formulation is clear and its practical implications are rich, \eqref{def: adversarial attack} is not crystal clear for the complexity of neural networks, non-convexity, and non-linearity, which hampers understanding of the effectiveness of adversarial training.

For these reasons, people proposed a toy model for \eqref{def: adversarial attack} in terms of two perspectives, which make the problem more tractable. First, instead of a parametric family, a solution space is the possibly largest one, the set of all Borel measurable probabilistic classifiers, denoted by $\mathcal{F}$,
\begin{equation}\label{def: solution space}
    \mathcal{F}:= \left\{ f=(f_i)_{i \in \mathcal{Y}}: \mathcal{X} \to \Delta_{\mathcal{Y}} : f \text{ is Borel measurable} \right\}.
\end{equation}
In the statistical learning community, the setting \eqref{def: solution space} is also called \emph{agnostic learning}, which refers to model-freeness \cite{agnostic_learning}. One can think of this as the \emph{convexification} of an original solution space. The second is to use an (extended) linear loss function
\begin{align*}
\label{eq:01loss}
    \ell(u, i) := 1 - u_i, \hbox{ for } (u,i) \in \Delta_{\mathcal{Y}} \times \mathcal{Y}.
\end{align*}
It is also referred as $0$-$1$ loss function since if $f(x) \in \mathcal{Y}$, then $\ell(f(x), i) = \delta_{\{f(x)=i\}}$. With these choices, \eqref{def: adversarial attack} is written as
\begin{equation}\label{eq:linear adversarial risk}
    \inf_{f \in \mathcal{F}}  R(T_\varepsilon f, \mu), \quad T_\varepsilon f_i(x) := \inf_{x' : d(x,x') \leq \varepsilon} f_i(x'),
\end{equation}
which has been intensively studied in machine learning community \cite{carlini2017magnet, athalye2018synthesizing, bose2018adversarial, Nakkiran2019AdversarialRM, Bhagoji2019LowerBO, pinot2020randomization, Meunier2021MixedNE, Pydi2021AdversarialRV, pydi2021the, awasthi2021existence, awasthi2021extended, frank2023the, frank2023existence, NEURIPS2023_81858558, jakwang_2024existence, frank2026notion}. However, this framework lacks mathematical rigor; see \cite{pydi2021the, jakwang_2024existence}

In this paper, we study the \emph{distributional-perturbing adversarial model}, which is the adversarial training model including \eqref{eq:linear adversarial risk} motivated by the distributional robust optimization framework in finance, and is defined as follows:
\begin{equation}\label{def: distributional model}
    \inf_{f\in \mathcal{F}} \sup_{\nu \in \mathcal{P}(\mathcal{X} \times \mathcal{Y})}\mathscr{R}(f,\nu; \mu), \quad  \mathscr{R}(f,\nu; \mu):= R(f,\nu) - C(\mu, \nu).
\end{equation}
where $R(\cdot, \cdot)$ is given in \eqref{eq:risk0}. Here, $C(\mu, \nu)$ is the \emph{transport cost} of an adversarial attack $\nu \in \mathcal{P}(\mathcal{X} \times \mathcal{Y})$, which is defined as
\begin{equation}\label{eq: general transport cost}
    C(\mu, \nu) = \inf_{\pi \in \Pi(\mu, \nu)} \int_{(\mathcal{X} \times \mathcal{Y})^2} \bar{c}( x,y, x',y' ) d\pi (x,y, x',y')
\end{equation}
where $\Pi(\mu, \nu)$ denotes the set of couplings whose marginals are $\nu$ and $\mu$.

\begin{remark}
For any finite positive measures $\alpha$ and $\beta$, $\Pi(\alpha, \beta)$ is defined in the same way. If $\alpha$ and $\beta$ have different total masses, however, $\Pi(\alpha, \beta)$ is empty, and the transport cost is defined to be infinite. 
\end{remark}

For $\mu \in \mathcal{P}(\mathcal{X} \times \mathcal{Y})$, we say that $f$ is a robust classifier for \eqref{def: distributional model} with input measure $\mu$ if $\mathscr{R}(f; \mu) = \mathscr{R}( \mu)$. Similarly, we say that $\mu$ is an optimal adversarial attack with input measure $\mu$ if $\mathscr{R}(\nu; \mu) = \mathscr{R}( \mu)$. Note that there are always a robust classifier and an optimal adversarial attack under a milder condition on $c$ (\Cref{assumption: cost function}): see \cite[Theorem 2.5]{jakwang_2024existence}.

With a slight abuse of notation, which should cause no confusion from the context, we write
\begin{align*}
\begin{aligned}
    \mathscr{R}(f;\mu)&:= \sup_{\nu \in \mathcal{P}(\mathcal{X} \times \mathcal{Y})}\mathscr{R}(f,\nu; \mu), \quad \mathscr{R}(\nu ; \mu) := \inf_{f\in \mathcal{F}} \mathscr{R}(f,\nu; \mu) \quad \hbox{and}\\
    \mathscr{R}(\mu)&:= \inf_{f\in \mathcal{F}} \sup_{\nu \in \mathcal{P}(\mathcal{X} \times \mathcal{Y})}\mathscr{R}(f,\nu; \mu).
\end{aligned}
\label{def: rdist}
\end{align*}

We assume the following structure on the transport cost:
\begin{equation}\label{def: transport cost}
    C(\mu, \nu) := \sum_{i \in \mathcal{Y}} C(\mu_i, \nu_i), \quad C(\mu_i, \nu_i) = \inf_{\pi_i \in \Pi(\mu_i, \nu_i)} \int_{\mathbb{R}^p \times \mathbb{R}^p} c(x, x') d\pi_i(x, x').
\end{equation}
Notice that \eqref{def: transport cost} is deduced from \eqref{eq: general transport cost} by choosing
\begin{equation}\label{eq: general cost}
    \bar{c}( x,y, x',y' ) = \begin{cases}
        c(x,x') &\text{ if $y=y'$,}\\
        \infty & \text{ otherwise.}
    \end{cases}
\end{equation}
The interpretation of \eqref{def: transport cost} is that any attack $\nu$ must satisfy $|\nu_i|=|\mu_i|$ for all $i \in \mathcal{Y}$, that is, \emph{the adversary is not allowed to attack the class distribution}.

%\newpage

\subsection{Remarks on cost functions}

The most popular cost function chosen in \eqref{eq: general cost} is the $0$-$\infty$ cost function parameterized by the adversarial budget $\varepsilon > 0$:
\begin{equation}
\label{def:CostEpsilon}
    c(x, x')= c_\varepsilon(x, x') := \begin{cases}
    \infty & \text{if } d(x, x') >\varepsilon, \\
    0 & \text{if } d(x',  x) \leq \varepsilon.
    \end{cases}
\end{equation}
For this case, \eqref{def: distributional model} and \eqref{eq:linear adversarial risk} are equivalent in the sense that both have the same minimax value, and share a robust classifier \cite{pydi2021the, jakwang_2024existence}.

While this formulation represents the popular adversarial perturbation model, it is shown that under general conditions (\Cref{assumption: cost function}) there are a robust classifier and an optimal adversarial attack, $0$-$\infty$ type of cost functions are too singular and lead to technical difficulties in order to establish desirable statistical and learning-theoretic properties: see \Cref{example: counter example}. Therefore, we consider cost functions that provide a regularized counterpart (\Cref{assumption: strict convexity of cost function}), which includes a distance-based cost function, $\left( \frac{d(x,x')}{\varepsilon} \right)^{p}$ for $p > 1$, that are frequently used in both theory and practice.

%\subsection{Example}
\Cref{example: counter example} demonstrates that it is not possible to develop the desired statistical properties of the adversarial models \eqref{def: distributional model} without additional assumptions. The obstructions arise for the following reasons: (1) the cost function \eqref{def:CostEpsilon} is too singular; and (2) the weak convergence of the empirical distribution to the population one is too weak. Note that the adversarial risk with the cost function \eqref{def:CostEpsilon} is not continuous in general.

\begin{figure}[t]
\centering
\resizebox{0.8\textwidth}{!}{%
\begin{tikzpicture}[
    scale=2.05,
    point/.style={circle,fill=black,inner sep=1.4pt},
    attack/.style={-{Latex[length=2mm]},thick,densely dashed},
    every node/.style={font=\small}
]

% ============================================================
% Left panel: limiting configuration
% ============================================================
\begin{scope}[xshift=-2.15cm]

    % Filled l1 balls
    \fill[blue!12]
        (0,0) -- (1,1) -- (0,2) -- (-1,1) -- cycle;
    \fill[orange!14]
        (0,0) -- (1,-1) -- (2,0) -- (1,1) -- cycle;
    \fill[green!13]
        (-2,0) -- (-1,-1) -- (0,0) -- (-1,1) -- cycle;

    % Boundaries of l1 balls
    \draw[blue!70!black,thick]
        (0,0) -- (1,1) -- (0,2) -- (-1,1) -- cycle;
    \draw[orange!85!black,thick]
        (0,0) -- (1,-1) -- (2,0) -- (1,1) -- cycle;
    \draw[green!55!black,thick]
        (-2,0) -- (-1,-1) -- (0,0) -- (-1,1) -- cycle;

    % Coordinate axes
    \draw[->,gray!70] (-2.25,0) -- (2.25,0)
        node[right] {$x_1$};
    \draw[->,gray!70] (0,-1.25) -- (0,2.25)
        node[above] {$x_2$};

    % Pairwise intersections
    \draw[very thick,purple]
        (0,0) -- (1,1)
        node[midway,below right=-1pt] {$B_{12}$};
    \draw[very thick,red!75!black]
        (0,0) -- (-1,1)
        node[midway,below left=-1pt] {$B_{13}$};

    % Support points
    \node[point,label=above right:{$a_1=(0,1)$}] at (0,1) {};
    \node[point,label=below right:{$a_2=(1,0)$}] at (1,0) {};
    \node[point,label=below left:{$a_3=(-1,0)$}] at (-1,0) {};

    % Common intersection
    \node[circle,fill=red,inner sep=1.8pt] at (0,0) {};
    \node[below=2pt] at (0,0)
        {$B_{123}=B_{23}=\{(0,0)\}$};

    % Adversarial transportation
    \draw[attack,blue!70!black] (0,1) -- (0.06,0.10);
    \draw[attack,orange!85!black] (1,0) -- (0.10,0.02);
    \draw[attack,green!55!black] (-1,0) -- (-0.10,0.02);

    % Ball labels
    \node[blue!70!black] at (0,1.72)
        {$\overline B_1(a_1)$};
    \node[orange!85!black] at (1.48,-0.35)
        {$\overline B_1(a_2)$};
    \node[green!55!black] at (-1.48,-0.35)
        {$\overline B_1(a_3)$};

    \node[font=\bfseries] at (0,-1.55)
        {(a) Limiting measure $\mu$};

\end{scope}

% ============================================================
% Right panel: perturbed configuration
% ============================================================
\begin{scope}[xshift=2.65cm]

    % Choose a visible representative perturbation eta=1/n
    \def\eta{0.35}

    % Filled l1 balls
    \fill[blue!12]
        (0,0) -- (1,1) -- (0,2) -- (-1,1) -- cycle;
    \fill[orange!14]
        (0,0) -- (1,-1) -- (2,0) -- (1,1) -- cycle;
    \fill[green!13]
        ({-2-\eta},0) --
        ({-1-\eta},-1) --
        ({-\eta},0) --
        ({-1-\eta},1) -- cycle;

    % Boundaries
    \draw[blue!70!black,thick]
        (0,0) -- (1,1) -- (0,2) -- (-1,1) -- cycle;
    \draw[orange!85!black,thick]
        (0,0) -- (1,-1) -- (2,0) -- (1,1) -- cycle;
    \draw[green!55!black,thick]
        ({-2-\eta},0) --
        ({-1-\eta},-1) --
        ({-\eta},0) --
        ({-1-\eta},1) -- cycle;

    % Coordinate axes
    \draw[->,gray!70] (-2.6,0) -- (2.25,0)
        node[right] {$x_1$};
    \draw[->,gray!70] (0,-1.25) -- (0,2.25)
        node[above] {$x_2$};

    % Only remaining overlap
    \draw[very thick,purple]
        (0,0) -- (1,1)
        node[midway,below right=-1pt] {$B_{12}$};

    % Support points
    \node[point,label=above right:{$a_1$}] at (0,1) {};
    \node[point,label=below right:{$a_2$}] at (1,0) {};
    \node[
        point,
        label=below left:
        {$\tilde{a}^n=(-1-\frac1n,0)$}
    ] at ({-1-\eta},0) {};

    % Origin
    \node[circle,fill=red,inner sep=1.8pt] at (0,0) {};
    \node[below right=1pt] at (0,0) {$(0,0)$};

    % Gap indicating isolation of class 3
    \draw[<->,thick,densely dotted]
        ({-\eta+0.02},-0.18) -- (-0.02,-0.18);
    \node[below] at ({-\eta/2},-0.18)
        {$\frac1n$};

    % Possible common attack for classes 1 and 2
    \coordinate (q) at (0.55,0.55);
    \node[circle,fill=purple,inner sep=1.5pt] at (q) {};
    \draw[attack,blue!70!black] (0,1) -- (q);
    \draw[attack,orange!85!black] (1,0) -- (q);

    % Class 3 remains in its own ball
    \draw[attack,green!55!black]
        ({-1-\eta},0) -- ({-1.55-\eta},0.25);

    % Ball labels
    \node[blue!70!black] at (0,1.72)
        {$\overline B_1(a_1)$};
    \node[orange!85!black] at (1.48,-0.35)
        {$\overline B_1(a_2)$};
    \node[green!55!black] at ({-1.55-\eta},0.65)
        {$\overline B_1(\tilde{a}^n_3)$};

    % Empty intersections
    \node[align=center] at (-1.25,-1.28) {
        %$B_{13}^n=B_{23}^n=B_{123}^n=\text{Var}nothing$
    };

    \node[font=\bfseries] at (0,-1.55)
        {(b) Perturbed measure $\mu^n$};

\end{scope}

\end{tikzpicture}%
}
\caption{
Geometry of the closed $\ell^1$-balls of radius $\varepsilon=1$.
In the limiting configuration, all three classes can be transported to
$(0,0)$. After replacing $a_3$ by
$\tilde{a}^n_3=(-1-\frac1n,0)$, the third ball is separated from the
other two, while classes $1$ and $2$ still interact along $B_{12}$.
}
\label{fig:counterexample-geometry}
\end{figure}
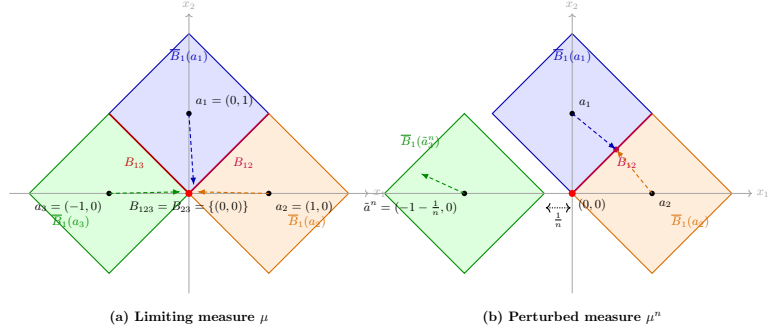

\begin{example}\label{example: counter example}
\normalfont
Let $(\mathbb{R}^p,d)= (\mathbb{R}^2, \| \cdot \|_1)$, $K=3$ and $\varepsilon=1$. Assume that $a_1 = (0,1), a_2=(1,0), a_3=(-1,0)$ and $\mu_i=\frac{1}{3} \delta_{a_i}$ for $i=1,2,3$. 

For the problem \eqref{def: distributional model} with the cost function \eqref{def:CostEpsilon}, observe that the geometry of feasible adversarial attacks is described as
\begin{align*}
    B_{12}:=\overline{B}_1(a_1) \cap \overline{B}_1(a_2) &= \{(x_1, x_2) : 0 \leq x_1=x_2 \leq 1 \},\\
    B_{13}:=\overline{B}_1(a_1) \cap \overline{B}_1(a_3) &= \{(x_1, x_2) : 0 \leq -x_1=x_2 \leq 1 \},\\
    B_{23}:=\overline{B}_1(a_2) \cap \overline{B}_1(a_3) &= \{(0,0) \}, \quad B_{123}:=\bigcap_{i=1}^3 \overline{B}_1(a_i) = \{(0,0)\}.
\end{align*}
From these observations, a unique optimal adversarial attack $\nu$ in \eqref{def: distributional model} is given by choosing $\nu_i = \frac{1}{3} \delta_{(0,0)}$ for all $i=1,2,3$, thereby collapsing all three classes to the same point and making them indistinguishable to the classifier.

On the other hand, the robust classifier against $\nu$ is given as follows: for each $i=1,2,3$, letting $j$ and $k$ are remaining indices,
\[
    f_i(x_1, x_2) :=
\begin{cases}
    1 &\text{ if $(x_1, x_2) \in \overline{B}_1(a_i) \setminus (B_{ij} \cup B_{ik})$},\\
    \frac{1}{2} &\text{ if $(x_1, x_2) \in (B_{ij} \cup B_{ik}) \setminus B_{123}$},\\
    \frac{1}{3} &\text{ if $(x_1, x_2) \in B_{123}$}.
\end{cases}
\]
$f_i(x_1, x_2)$ is the reciprocal of the number of classes that may appear by the adversary at $(x_1, x_2)$. %Similarly, $f_2$ and $f_3$ are given as
%\[
%    f_2(x_1, x_2) =
%\begin{cases}
%    1 &\text{ if $(x_1, x_2) \in B_1(a_2)$},\\
%    \frac{1}{2} &\text{ if $(x_1, x_2) \in B_{12}$},\\
%    \frac{1}{3} &\text{ if $(x_1, x_2) \in B_{123}$},
%\end{cases}
%\quad 
%f_3(x_1, x_2) =
%\begin{cases}
%    1 &\text{ if $(x_1, x_2) \in B_1(a_3)$},\\
%    \frac{1}{2} &\text{ if $(x_1, x_2) \in B_{13}$},\\
%    \frac{1}{3} &\text{ if $(x_1, x_2) \in B_{123}$}.
%\end{cases}
%\]
Note that outside of $\bigcup \overline{B}_1(a_i)$, $f$ can be defined arbitrarily since it is irrelevant to the adversarial risk. The optimal adversarial risk achieved by a pair $(f,\nu)$ is $\frac{2}{3}$; at $(0,0)$, the probability that $f$ predicts a correct class is probability $\frac{1}{3}$.

Now, consider $\tilde{a}^n_3 =(-1 - \frac{1}{n}, 0)$, $\mu_3^n = \frac{1}{3} \delta_{\tilde{a}^n_3}$ and $\mu^n = \mu_1 + \mu_2 + \mu_3^n$. For each $n \in \mathbb{N}$, let
\begin{align*}
    B^n_{12} &= B_{12}, \quad B^n_{13}:=\overline{B}_1(a_1) \cap \overline{B}_1(\tilde{a}^n_3) = \emptyset,\\
    B^n_{23}&:=\overline{B}_1(a_2) \cap \overline{B}_1(\tilde{a}^n_3) = \emptyset, \quad B^n_{123}:=%\bigcap_{i=1}^3 
    \overline{B}_1(a_1) \cap \overline{B}_1(a_2) \cap \overline{B}_1(\tilde{a}^n_3) = \emptyset.
\end{align*}
For each $n$, one can characterize the set of optimal adversarial attacks as follows: $\nu^n = (\nu_1^n, \nu_2^n, \nu_3^n)$ is an optimal adversarial attack if
\[
    \nu_1^n=\nu_2^n \text{ such that $\spt(\nu_1^n)=\spt(\nu_2^n) \subseteq B_{12}$ and any } \nu_3^n \text{ with $\spt(\nu_3^n) \subseteq \overline{B}_1(\tilde{a}^n_3)$}.
\]
In words, the adversary can only perturb classes $1$ and $2$ by moving mass identically to a subset of $B_{12}$ equally likely, while the adversary cannot attack class $3$ since class $3$ is isolated in the sense that the support of $\mu_3^n$ is far from the supports of other classes.

There are two types of robust classifiers given $\nu^n$. The first type is $f_1=f_2=\frac{1}{2}$ on $B_{12}$. On any subset of $B_{12}$, the adversary is able to give some mass to class $1$ and class $2$ equally, which forces the learner to guess the two classes equally. On $\overline{B}_1(a_1) \setminus B_{12}$ and $\overline{B}_1(a_2) \setminus B_{12}$, $f_1=1$ and $f_2=1$, respectively. On $\overline{B}_1(\tilde{a}^n_3)$, whatever the adversary does, the learner always knows that there is only class $3$, hence $f_3=1$ is optimal.

The second type is either $f_1=1, f_2=0$ or $f_1=0, f_2=1$ on $B_{12}$. It is shown in \cite{bungert2023geometry} that there is always an optimal set classifier for the binary setting. Since class $3$ does not interact with other classes, the situation reduces to binary classification restricted to $\overline{B}_1(a_1) \cup \overline{B}_1(a_2)$. Then, $f_1=1$ on $\overline{B}_1(a_1)$ and $f_2=1$ on $\overline{B}_1(a_2) \setminus B_{12}$. $f_3$ is defined the same as in the previous case.

It turns out that for any $n \in \mathbb{N}$, the optimal adversarial risk is $\frac{1}{3}$, which shows that the adversarial risk is generally not continuous. In fact, it is upper semicontinuous with respect to the input measure with this singular cost function; see \Cref{cor: continuity of attack} and \Cref{rmk: upper semicontinuity}.

Regarding the convergence of classifiers, the issue arises from their behaviors at the point $(0,0)$. Since the support of class $3$ is further than distance $2$ from the supports of classes $1$ and $2$, especially $d(\spt(\mu^n_3), (0,0)) > 1$ for all $n$, any robust classifier should assign $f_3=0$ at $(0,0)$; otherwise, it achieves the risk strictly larger than $\frac{1}{2}$ at $(0,0)$. Any robust classifier for $\mu$, however, should assign $(f_1, f_2, f_3)=(\frac{1}{3}, \frac{1}{3}, \frac{1}{3})$ at $(0,0)$; otherwise, the adversary can always deviate his/her action to increase the risk. Therefore, any sequence of robust classifiers $\{f^n\}$ fails to converge to $f$ at $(0,0)$, which is an obstruction to convergence.
\end{example}

\subsection{Notation}
\begin{itemize}
    \item Given a measure $\xi$ and a function $f$, we sometimes use
    \[
        \xi(f):= \int f d \xi.
    \]
    \item For $g=(g_i)_{i \in \mathcal{Y}} : \mathbb{R}^p \to \mathbb{R}^K$, let
\begin{equation}\label{eq: extension of g_i}
    \bm{g}(x,y) = \sum_{i \in \mathcal{Y}} g_i(x) \mathds{1}_{y=i},
\end{equation} 
which is a function on $\mathcal{X} \times \mathcal{Y}$.
    \item $\mathcal{N}(a, b^2)$ denotes the Gaussian distribution with the mean $a$ and the variance $b^2$.
\end{itemize}

\section{Main results}
\label{sec:main}

The lessons from \Cref{example: counter example} illustrate that additional regularity assumptions are indispensable for establishing the fundamental statistical properties of the model. At an intuitive level, the instability of \eqref{def: distributional model} can be traced to the singularity of the cost function $c$ in \eqref{def:CostEpsilon}.

There is an elementary choice of a good approximation of a singular cost $c$ in \eqref{def:CostEpsilon} that possesses sufficient regularity as follows. Given a fixed adversarial budget $\varepsilon > 0$, define
\[
    c_{\varepsilon, p}(x,x'):=\left( \frac{d(x,x')}{ \varepsilon} \right)^p
\]
for some $p > 1$. With such a cost function, the corresponding transport cost is the $p$-Wasserstein distance with scaling $\varepsilon^{-p}$. In fact, this choice is a reasonable approximation, which is justified by 
\[
    \lim_{p \to \infty } c_{\varepsilon, p}(x,x') = \begin{cases}
        0 \text{ if $d(x,x') < \varepsilon$,}\\
        1 \text{ if $d(x,x') = \varepsilon$,}\\
        \infty \text{ if $d(x,x') > \varepsilon$.}
    \end{cases}
\]
Hence, as $p \to \infty$, it recovers $c_\varepsilon$ of \eqref{def:CostEpsilon} generically.

Here, a larger class of cost functions including $p$-th power of distance for $p >1$ is assumed. Optimal transport with such regular cost functions has been studied and pioneered by \citet{gangbo1996geometry}. Key features of such cost functions are \emph{strict convexity}, \emph{cone geometry}, and \emph{superlinearity}.

\begin{assumption}\label{assumption: strict convexity of cost function}
$c(x,x')=h(x-x')$, where $h:\mathbb{R}^p \rightarrow [0, \infty)$ with $h(0)=0$ is a symmetric and continuous function satisfying
\begin{enumerate}
	\item[(A1)] $h$ is strictly convex on $\mathbb{R}^p $,
	\item[(A2)] given a height $r\in \R^+$ and an angle $\theta \in (0,\pi) $, there exists some $M:=M(r, \theta)>0$ such that for all $|p |>M$, one can find a cone 
	\begin{align*}
		K(r, \theta, z,p):=\left\lbrace x\in \mathbb{R}^p  : | x-p|| z|\cos(\theta/2)\leq \left< z,x-p \right>\leq  r| z| \right\rbrace,
	\end{align*}
	with vertex at $p$ on which $h$ attains its maximum at $p$,
	\item[(A3)] $\lim_{|x | \rightarrow \infty}\frac{h(x)}{|x | }= \infty $.
\end{enumerate}    
\end{assumption}

Now, we are ready to state our first CLT.

\begin{theorem}[CLT centered at empirical adversarial risk]\label{thm: empirical centering}
Assume that $c$ satisfies \Cref{assumption: strict convexity of cost function}, and there is a unique optimal $g=(g_i)_{i \in \mathcal{Y}}$ for \eqref{eq:mot_decomposed_dual} with input $\mu$. Recall \eqref{eq: extension of g_i}. Then,
\[
    \sqrt{n}(\mathscr{R}(\mu^n)- \mathbb{E}\mathscr{R}(\mu^n) ) \overset{d}{\longrightarrow} \mathcal{N} \left( 0, \mu(\bm{g}^2) - (\mu(\bm{g}) )^2 \right)
\]
where
\[
    \mu(\bm{g}^2) - (\mu(\bm{g}) )^2 = \sum_{i \in \mathcal{Y}} \int g_i^2 d\mu_i - \sum_{i,  j} \int g_i d\mu_i \int g_j d\mu_j.
\]
\end{theorem}    

\begin{proof}
See \Cref{thm: CLT empirical}.    
\end{proof}

\Cref{thm: empirical centering} is only partially satisfactory, as it is centered at the expected empirical adversarial risk rather than the population risk, leaving a potential bias. Since obtaining a population-centered CLT directly is challenging, we instead consider the smoothed adversarial training risk. Smoothing also has a natural interpretation: a labeled image $(x,y)$ is observed after unknown noise is added to $x$.

Smoothing is also widely used in optimal transport to obtain population-centered CLTs \cite{Goldfeldetal_2020, limit_pWasserstein2024, Goldfeld_Kato_Rioux_Sadhu2024}. With a sufficiently smooth kernel, it provides the regularity needed to transform the dual class, via the Fubini theorem, into a sufficiently regular function class to which classical empirical process theory applies, yielding a limiting Gaussian process.

These existing results, however, do not directly apply here since our function class is genuinely vector-valued. This difficulty is resolved by lifting a finite-dimensional vector-valued potential to a real-valued function on a natural extended space. With the boundedness of potentials, one can show that a lifted function class is $\mu$-Donsker. Combined with $L^2$-convergence of optimal potentials, the population-centered CLT is achieved.

Let us formally introduce the smoothed measure. Given a sequence of measures $\{\mu^n\}$ and a measure $\mu$, we denote their smoothed version by 
\[
    \mu^{n, \chi}:=\mu^n * \chi, \quad \mu^{\chi}:=\mu * \chi.
\]
They should be understood as
\[
    \mu^{\chi}(X \in A, Y =i) = \int_A d (\mu_i * \chi)(x).
\]
Popular choices for $\chi$ include a centered Gaussian measure or a compactly supported probability measure with a smooth density.

\begin{assumption}\label{assumption: smoothing kernel}
The smoothing measure $\chi$ is a symmetric probability measure that is
absolutely continuous with respect to the Lebesgue measure, with density
$\kappa$, and
\[
    \kappa\in W^{m,1}(\mathbb R^p)
\]
for some integer $m>p/2$. In addition, $\chi$ has negligible boundary.
\end{assumption}

%\dk{Example 넣어줘야함; Gaussian + bounded support}

The next one is the assumption about the tail probability of $\mu_X$, which is necessary for the existence of a Gaussian process.

\begin{assumption}\label{assumption: measure for empirical process}
There is a partition $\{\mathcal{X}_j\}$, which are uniformly bounded and convex sets with nonempty interior, of $\mathbb{R}^p$ such that 
\[
   \sum_j \mu_X (\mathcal{X}_j)^{\frac{1}{2}} < \infty. 
\]
for the marginal $\mu_X$ given in \eqref{eq:marg}.
\end{assumption}

\begin{remark}
A similar assumption is imposed in \cite[Proposition 3.1]{limit_pWasserstein2024}, which suffices to guarantee a class of sufficiently regular functions to be $\mu$-Donsker. For example, if $\mu$ is subexponential or compactly supported, the assumption is satisfied.
\end{remark}

Now, we have the CLT centered at the population smoothed adversarial risk.

\begin{theorem}[CLT of the smoothed adversarial risk]\label{thm: smooth limit}
Assume that $c$ satisfies \Cref{assumption: strict convexity of cost function}, $\mu_X$ satisfies \Cref{assumption: measure for empirical process}, the supports of $\mu_i$ have negligible boundary, and $\chi$ satisfies \Cref{assumption: smoothing kernel}. Assume further that there is a unique optimal $g=(g_i)_{i \in \mathcal{Y}}$ for \eqref{eq:mot_decomposed_dual} with input $\mu * \chi$. Let
\begin{align}
\label{eq:gchi}
\bm{g}_\chi(x,y) := \sum_{i \in \mathcal{Y}}  (g_i * \chi) (x) \mathds{1}_{y=i}.
\end{align}
Then, 
\[
    \sqrt{n} (\mathscr{R}(\mu^{n,\chi})- \mathscr{R}(\mu^\chi) ) \overset{d}{\longrightarrow} \mathcal{N} \left( 0,  \mu( \bm{g}_\chi^2) - \left( \mu(\bm{g}_\chi) \right)^2 \right)
\]
where 
\[
    \mu( (\bm{g}_\chi)^2) - \left( \mu(\bm{g}_\chi) \right)^2 = \sum_{i \in \mathcal{Y}} \mu_i \left( (g_i * \chi)^2 \right) - \sum_{i,  j} \mu_i(g_i * \chi)\mu_j(g_j * \chi).
\]
\end{theorem}

\begin{proof}
See \Cref{thm: limit distribution for smoothed}.    
\end{proof}

A significant difference from the asymptotic theory of standard empirical optimal transport is the presence of nonvanishing cross-covariance terms. Because of these terms, uniqueness of the optimal potential only up to additive constants is not sufficient to identify the limiting variance.

To see this, consider the binary setting. Suppose that two optimal potentials \(g\) and \(g'\) differ by an additive transformation,
\[
    g_1'=g_1+a, \quad g_2'=g_2-a,
\]
for some $a\in\mathbb{R}$. Recalling the variance of the Gaussian limit in \Cref{thm: empirical centering}, we have
\begin{align*}
    &\mu_1( (g_1')^2) + \mu_2( (g_2')^2) - (\mu_1 (g_1') + \mu_2(g_2') )^2\\
    &=\mu_1( (g_1+ a)^2) + \mu_2( (g_2 - a)^2) - (\mu_1 (g_1 + a) + \mu_2(g_2 -a) )^2\\
    &\neq \mu_1( g_1^2) + \mu_2( g_2^2) - (\mu_1 (g_1) + \mu_2(g_2) )^2
\end{align*}
in general. Thus, two potentials that are equivalent up to additive constants need not yield the same limiting variance, so the variance is not well defined on the corresponding equivalence class.

This phenomenon does not arise in the standard empirical optimal transport setting considered in \cite{Tameling_etc2019, CLT_OT2019, Goldfeld_Kato_Rioux_Sadhu2024, limit_pWasserstein2024, CLT_semidiscrete_OT2024, CLT_general_transport2024, Hundrieser_etc2024, rodriguezvitores2025improvedcentrallimittheorem, delbarrio2025distributionallimittheoryoptimal}, where the empirical marginals are sampled independently and the corresponding cross-covariance terms are absent. This distinction shows that genuine uniqueness of the optimal potential is required to identify the limiting variance, and hence to establish the CLT for the adversarial training model.

We establish uniqueness of the optimal potential in the binary setting, under which the adversarial problem is equal to optimal partial transport. Unlike standard optimal transport, where potentials are unique only up to additive constants, the optimal potential is genuinely unique for non-degenerate optimal partial transport.

The proof exploits the equivalence among adversarial training, its multimarginal optimal transport formulation and the dual of it, and optimal partial transport. By reducing partial transport to standard transport through an extension of the input measures, we apply known uniqueness results for optimal transport potentials. Complementary slackness then determines the potential on the inactive region, yielding uniqueness almost everywhere.

We further show that, in the general multiclass setting, the optimal adversarial attack is unique when the input measures are absolutely continuous, under \Cref{assumption: strict convexity of cost function}. Note that the following uniqueness statements are understood to hold $\mu$-almost everywhere.

\begin{theorem}[Uniqueness]\label{thm: uniqueness}
Assume that $c$ satisfies \Cref{assumption: strict convexity of cost function}, and $\mu_i$'s are absolutely continuous measures with connected supports. Then, there is a unique optimal adversarial attack for \eqref{def: distributional model}.

Consider the binary setting. There is a unique optimal partition for \eqref{def: stratified system of barycenter problem}. Furthermore, the following holds.
\begin{enumerate}
    \item[(i)] If $\mu_{1,1}=\mu_{2,2}=0$ (degenerate case), optimal potential for \eqref{eq:mot_decomposed_dual} is unique up to additive constants, i.e., if $g$ and $g'$ are optimal, there is a constant $a \in \mathbb{R}$ such that
\[
    g_1' = g_1 + a, \quad g_2'=g_2 - a.
\]
    \item[(ii)] If either $\mu_{1,1}$ or $\mu_{2,2}$ does not vanish (non-degenerate case), then optimal $g$ is unique.
\end{enumerate}    
\end{theorem}

\begin{proof}
See \Cref{lem: uniqueness of adversarial attack} and \Cref{thm: uniquness of potential}.    
\end{proof}

\begin{remark}
Using the $\overline{c}$-transform defined in \eqref{eq:c_transform}, it is immediate that under the hypothesis of \Cref{thm: uniqueness}, there is a unique robust classifier for \eqref{def: distributional model} for non-degenerate case.
\end{remark}

Leveraging \Cref{thm: uniqueness}, one can derive the following CLTs for the binary setting.

\begin{corollary}\label{Cor: empirical centering}
Consider the binary setting. Assume that $c$ satisfies \Cref{assumption: strict convexity of cost function}, and $\mu_1$ and $\mu_2$ are absolutely continuous. Assume further that the non-degenerate case holds as in \Cref{thm: uniqueness} with input $\mu$. Then, there is a unique optimal potential $(g_1, g_2)$ for \eqref{eq:mot_decomposed_dual} with input $\mu$, and
\[
\sqrt n\left(
\mathscr R(\mu^n)-\mathbb E\mathscr R(\mu^n)
\right)
\overset{d}{\longrightarrow}
\mathcal N(0,\sigma^2),
\]
where
\[
\sigma^2
=
\mu_1(g_1^2)+\mu_2(g_2^2)
-\left(\mu_1(g_1)+\mu_2(g_2)\right)^2.
\]
\end{corollary}

\begin{corollary}\label{Cor: smoothed centering}
Consider the binary setting. Assume that $c$ satisfies \Cref{assumption: strict convexity of cost function}, $\mu_1$ and $\mu_2$ satisfy \Cref{assumption: measure for empirical process}, and have negligible boundary. Also, assume that $\chi$ satisfies \Cref{assumption: smoothing kernel}, and the non-degenerate case holds as in \Cref{thm: uniqueness} with input $\mu * \chi$. Then, there is a unique optimal potential $(g_1, g_2)$ for \eqref{eq:mot_decomposed_dual} with input $\mu * \chi$, and
\begin{align*}
    \sqrt{n}(\mathscr{R}(\mu^{n, \chi})- \mathscr{R}(\mu^{\chi}) )   \overset{d}{\longrightarrow} \mathcal{N}\left(0, \sigma_\chi^2 \right),
\end{align*}
where
\[
\sigma_\chi^2
=
\mu_1((g_1 * \chi)^2)+\mu_2((g_2 * \chi)^2)
-\left(\mu_1(g_1 * \chi)+\mu_2(g_2 * \chi)\right)^2.
\]
\end{corollary}

As a byproduct of the CLTs, we establish the stability of optimal robust classifiers and attacks. As an empirical measure $\mu^n \to \mu$ weakly, there is a subsequence of saddle points for \eqref{def: distributional model} indexed by input measures $\mu^n$ converging to the limit, which is a saddle point with input $\mu$. A key idea is the duality of \eqref{def: distributional model}, and the Lipschitz regularity of the classifier and the associated optimal potential.

\begin{theorem}[Stability]\label{thm: convergence}
Assume that $c$ satisfies \Cref{assumption: strict convexity of cost function} and $\mu^n \to \mu$ weakly. For each $n \in \mathbb{N}$, there are an optimal classifier $f^n$ and an optimal adversarial attack $\nu^n$ for \eqref{def: distributional model} with input measure $\mu^n$ given in \eqref{eq:empdis}. Furthermore, there are convergent subsequences of $f^n$ and $\nu^n$ such that they converge to $f$ pointwise and $\nu$ weakly, respectively, and $f$ and $\nu$ are optimal for \eqref{def: distributional model} with input measure $\mu$. 
\end{theorem}

\begin{proof}
\Cref{prop: stability of optimal potentials} with \Cref{thm:learner_part} and \Cref{cor: continuity of attack}.    
\end{proof}

Another byproduct of our main results is about the generalization error of the adversarial training model. For this, we need to introduce the $W_1$ distance (a.k.a. Kantorovich-Rubinstein distance). Given two probability measures $\mu, \nu \in \mathcal{P}(\mathcal{S})$ over a complete metric space $(\mathcal{S}, d_\mathcal{S})$ with finite first moments, $W_1$ distance is defined as
\[
    W_1(\mu,\nu) := \inf_{\pi \in \Pi(\mu, \nu)} \int d_\mathcal{S}(s_1, s_2) d\pi(s_1, s_2)
\]
It is well known that its dual is
\[
    \sup \left\{ \int \varphi d\sigma : \varphi \in \cap L^1(d |\sigma|), \| \varphi \|_{Lip} \leq 1 \right\}.
\]

On $\mathcal{X} \times \{1, \dots, K\}$, define a metric
\[
    d( (x,y), (x',y'):= d_\mathcal{X}(x,x') + \mathds{1}_{y \neq y'}.
\]
Let $W_{1}$ denote the
corresponding Wasserstein distance with the above $d$.

We show that the generalization error of the adversarial training model is bounded by the $W_1$ distance between $\mu^n$ and $\mu$. Again, a key observation is the Lipschitz regularity of the associated optimal potentials, by which the generalization error and $W_1$ distance are connected.

Once arrive at this point, combining this with recent progress on the convergence rates of empirical measures in Wasserstein distance (see \Cref{thm : upper bound of expected distance} and \Cref{lem:concentration}), one can also obtain the sample complexity and the tail probability of the generalization error.

\begin{theorem}[Generalization error]\label{thm: generalization error}
Assume $c$ satisfies \Cref{assumption: strict convexity of cost function}. Let $f^n$ and $f^*$ be robust classifiers for \eqref{def: distributional model} with inputs $\mu^n$ and $\mu$, respectively. Let
\[
    \Delta(f^n;\mu):=\mathscr{R}(f^n; \mu) - \mathscr{R}(f^*; \mu).
\]
Then, 
\[
   0 \leq \Delta(f^n;\mu) \leq 2 (2 \vee L_c) W_1(\mu, \mu^n)
\]
where $L_c > 0$ is a constant depending on $c$.

Furthermore, the following holds.
\begin{enumerate}
    \item[(i)] If $\mu$ has a finite $3$-rd moment, i.e. $M_3(\mu):=||X||_{L^3(\mu)} < \infty$, then
\begin{equation*}
    \begin{aligned}
    0 \leq \mathbb{E}\Delta(f^n;\mu) \leq 20 d (2 \vee L_c) M_3(\mu)  
    \left\{ 
    \begin{array}{ll}
    n^{-\frac{1}{2}} & \textrm{if $d=1$,}\\
    n^{-\frac{1}{2}}\sqrt{\log (1 + n)} & \textrm{if $d=2$,}\\
    n^{-\frac{1}{d}} & \textrm{if $d \geq 3$.}
    \end{array} 
    \right.
    \end{aligned}
\end{equation*}
Furthermore, for any $n \geq 1, t \in (0,\infty)$ and for each $r \in (0, 3)$
\begin{align*}
    &\mathbb{P}\left( \Delta(f^n;\mu) \geq t \right)\\
    &\leq C
    n (nt)^{-(3 -r)} + C \mathds{1}_{\{t\leq 1\}}
    \left\{\begin{array}{ll}
    \exp(-cnt^2) & \hbox{if $d=1$}, \\[+3pt]
    \exp(-cn(t/\log(2+1/t))^2) & \hbox{if $d=2$}, \\[+3pt]
    \exp(-cn t^{d}) & \hbox{if $d \geq 3$}.
\end{array}\right.
\end{align*}
Here, the positive constants $C$ and $c$ depend only on $d$, $M_3(\mu)$ and $r$.

    \item[(ii)] Assume that $\mu$ has a bounded support with diameter $D$, and $k > d^*_1(\mu) \vee 2$ where $d^*_p(\mu)$ is the upper $p$-Wasserstein dimension for $\mu$. Then there exist constants $C=C(k) > 0$ such that
\begin{align*}
    \mathbb{E} \Delta(f^n;\mu) \leq 2 (2 \vee L_c) D^2  \left(3^{\frac{3k}{k - 2} + 1}\left( \frac{1}{3^{\frac{k}{2} - 1} - 1} + 3 \right) n^{- \frac{1}{k}} + C^{\frac{k}{2}} n^{- \frac{1}{2}} \right). 
\end{align*}
Furthermore, for any $n \geq 1, t \in (0,\infty)$,
\begin{align*}
    &\mathbb{P} \left(  \Delta(f^n;\mu) \geq \mathbb{E}  \Delta(f^n;\mu) + t \right) \leq \exp\left(-2n\frac{t^2}{4D^2} \right).
\end{align*}
\end{enumerate}    
\end{theorem}

\begin{proof}
\Cref{prop: generalization error for Lischitz cost} and \Cref{cor: sample complexity of generalization error}.    
\end{proof}

\section{Preliminaries}

\subsection{Generalized barycenters, MOT reformulations, and dual}

In this subsection, we collect the optimal transport formulations of the adversarial training problem that will be used in most of the subsequent proofs. We describe the equivalence among the distributional adversarial model, the generalized barycenter problem, the stratified multimarginal optimal transport (MOT) formulation, and the corresponding dual problem. We refer the reader to \cite{jakwang_MOT, jakwang_2024existence} for further technical details.

\citet{jakwang_MOT} show that \eqref{def: distributional model} is equivalent to a \emph{generalized barycenter problem}, which extends the classical Wasserstein barycenter formulation. A key feature of this reformulation is that it localizes the learner’s optimization, thereby reducing the original minimax problem to a maximization problem for the adversary—or, equivalently, to a minimization problem after a change of sign. The resulting formulation is more tractable than the original minimax problem.

To establish the existence of solutions, we impose the following conditions on the cost function c. Notably, these conditions allow for the singular cost defined in \eqref{def:CostEpsilon}.

\begin{assumption}\cite[Assumption 1]{jakwang_MOT}\label{assumption: cost function}
The cost $c: \X \times \X \rightarrow[0, \infty]$ satisfies the following:
\begin{itemize}
    \item lower semicontinuous, symmetric and $c(x,x)=0$ for all $x\in \X$.
    \item coercivity, i.e., if $\{ x_n \}$ is a bounded sequence, and $\{ x_n' \}$ is a sequence such that $\sup_{n \in \mathbb{N}} c(x_n, x_n')< \infty$,
    then $\{(x_n, x_n') \}_{n \in \mathbb{N}}$ is relatively compact in $\X \times \X$ (endowed with the product topology).
\end{itemize}
\end{assumption}
From now on, we always assume that $c$ satisfies \Cref{assumption: cost function}.

\begin{remark}
If $c$ satisfies \Cref{assumption: strict convexity of cost function}, then it satisfies \Cref{assumption: cost function}. 
\end{remark}

Define the generalized barycenter problem
\begin{equation}\label{def: generalized barycenter problem}
    \sup_{\lambda, \nu_1, \ldots, \nu_i} \left\{ 1 - \left( |\lambda| + \sum_{i \in \mathcal{Y}}C(\mu_i, \nu_i) \right)  : \text{ $\lambda \geq \nu_i$ for all $i \in \mathcal{Y}$} \right\}.
\end{equation}
Then $\eqref{def: distributional model} = \eqref{def: generalized barycenter problem}$, and furthermore the infimum of \eqref{def: generalized barycenter problem} is attained (\cite[Proposition 7 and Corollary 32]{jakwang_MOT}).

One can reduce \eqref{def: generalized barycenter problem} to the \emph{stratified} barycenter problem, which plays an important role in connecting adversarial training models to MOT (\cite[Proposition 11]{jakwang_MOT}).

Let 
\begin{equation}\label{eq: interaction set}
    S_\mathcal{Y} :=\{A \subseteq \mathcal{Y} : A \neq \emptyset \}, \quad S_\mathcal{Y}(i) :=\{A \subseteq \mathcal{Y} : i \in A \}.
\end{equation}
It is shown that \eqref{def: generalized barycenter problem} is equivalent to
\begin{equation}\label{def: stratified system of barycenter problem}
    \sup_{ \{\lambda_A, \mu_{i,A}\}} \left\{ 1 - \sum_{A \in S_\mathcal{Y}}\left( |\lambda_A| + \sum_{i \in A}C(\mu_{i,A}, \lambda_A) \right)  : \sum_{A \in S_\mathcal{Y}(i)} \mu_{i, A} = \mu_i \text{ for all $i \in \mathcal{Y}$} \right\}.
\end{equation}

As stated above, \eqref{def: stratified system of barycenter problem} is equivalent to the system of stratified MOT problems, which is the generalization of the equivalence between the classical Wasserstein barycenter problem and a single MOT \cite{Carlier2010MatchingFT, Agueh_Carlier2011, Carlier2015}. This equivalence enables explicit computation of the generalized barycenter, and consequently of \eqref{def: distributional model}, via solving the corresponding MOT (\cite[Propositions 14 and 15]{jakwang_MOT}).

For each $A \in S_\mathcal{Y}$, let $x_A:=(x_i)_{i \in A}$, and define $c_A: (\mathbb{R}^p)^A \to [0,\infty]$ as $c_A(x_A):=\inf_{x' \in \mathbb{R}^p} \sum_{i\in A} c(x', x_i)$. Consider the problem:
\begin{equation}\label{eq:multimarginal_decomposed}
\begin{aligned}
    &\sup_{ \{ \pi_A : A \in S_\mathcal{Y}\} } \left\{ 1 - \sum_{A\in S_\mathcal{Y}} \int_{(\mathbb{R}^p)^A} \left( 1 + c_A(x_A) \right) d\pi_A(x_A) \right\}\\
    &\textup{s.t.} \sum_{A\in S_\mathcal{Y}(i)}\mathcal{P}_{i\,\#}\pi_A=\mu_i  \textup{ for all } i\in \mathcal{Y},
\end{aligned}
\end{equation}
where $\mathcal{P}_i$ is the projection map $\mathcal{P}_i: (x_1, \ldots, x_K)\mapsto x_i$. Then \eqref{def: generalized barycenter problem} = \eqref{eq:multimarginal_decomposed}. Also, the supremum in \eqref{eq:multimarginal_decomposed} is attained.

\begin{remark}
Once \eqref{eq:multimarginal_decomposed} is solved, the optimal adversarial attacks can be obtained as follows: if $\{\pi^*_A : A \in S_\mathcal{Y} \}$ is optimal for \eqref{eq:multimarginal_decomposed}, then an optimal adversarial attack $(\nu^*_i)_{i}$ is constructed by
\[
    \nu^*_i =  \sum_{A\in S_\mathcal{Y}(i)}(T_{A})_\# (\pi^*_A)
\]
where $T_A(x_A) : x_A \mapsto \argmin \sum_{i\in A} c(x', x_i)$ is any measurable map.    
\end{remark}

The MOT problem admits a dual formulation naturally interpreted as the learner’s problem, since its primal corresponds to the adversary’s problem. Dual potentials and robust classifiers are linked through the $\overline{c}$-transform and $c$-transform. This correspondence makes MOT duality central to the statistical and learning-theoretic analysis of \eqref{def: distributional model}.

Let us introduce $c$-transform (resp. $\overline{c}$-transform): given functions $f$ and $g$, the $c$-transform of $f$
 and $\overline{c}$-transform of $g$ are given as
 \begin{equation}\label{eq:c_transform}
     f^c(x) := \inf_{x'} \left\{ f(x') + c(x,x') \right\}, \quad g^{\overline{c}}(x):= \sup_{x'} \left\{ g(x') - c(x,x') \right\}.
 \end{equation}
Note that $f^c$ and $g^{\overline{c}}$ are Borel measurable if $c$ is continuous.

\begin{lemma}{\cite[Propositions 22 and 24]{jakwang_MOT}}
\label{thm:learner_part}
Assume that $c$ satisfies \Cref{assumption: cost function}. Consider
\begin{equation}\label{eq:barycenter_dual}
\begin{aligned}
    \inf_{f \in \mathcal{F}} \sum_{i \in \mathcal{Y}} \int (1-f_i^c(x_i)) d\mu_i(x_i),
\end{aligned}
\end{equation}
where $\mathcal{F}$ is defined as \eqref{def: solution space}, and
\begin{equation}\label{eq:mot_decomposed_dual}
\begin{aligned}
    \inf_{g=(g_i)_{i \in \mathcal{Y}} \in \mathcal{G}} \sum_{i \in \mathcal{Y}} \int (1-g_i(x_i)) d\mu_i(x_i)    
\end{aligned}
\end{equation}
where
\begin{equation}\label{def: dual potential space}
    \mathcal{G}:= \left\{g \in C_b^K: \sum_{i\in A} g_i(x_i)\leq 1+c_A(x_A)  \textup{ for all }  x_A \in \mathcal{X}^A, A \in S_\mathcal{Y} \right\},
\end{equation}
respectively (recall \eqref{eq: interaction set} for $S_\mathcal{Y}$). Then,
\[
    \eqref{def: distributional model} = \eqref{eq:barycenter_dual} =  \eqref{eq:mot_decomposed_dual}.
\]
In particular, it is sufficient to restrict $0 \leq f_i, g_i \leq 1$.

If $f$ and $g$ are feasible for \eqref{eq:barycenter_dual} and \eqref{eq:mot_decomposed_dual} respectively, then
\[
    f'_i := g_i^{\overline{c}}, \quad g'_i := f^c_i
\]
are feasible for \eqref{eq:barycenter_dual} and \eqref{eq:mot_decomposed_dual}, respectively. In particular, if there is an optimal solution $f$ for \eqref{eq:barycenter_dual}, then it is also optimal for \eqref{def: distributional model}.

%If, in addition, the cost function $c$ is bounded and Lipschitz, then the optimal value of \eqref{eq:barycenter_dual} and \eqref{eq:mot_decomposed_dual} is achieved by a bounded and Lipschitz function with the same Lipschitz constant of $c$. 
\end{lemma}

\subsection{Optimal partial transport}\label{sec: optimal partial transport}
Recalling the MOT formula \eqref{eq:multimarginal_decomposed} for the binary setting, a straightforward calculation yields
\begin{align*}
    &1 - \sum_{A\in S_\mathcal{Y}} \int \left( 1 + c_A(x_A) \right) d\pi_A(x_A)\\
    &= (|\mu_1| - |\pi_1| )+ (|\mu_2|  - |\pi_2|) - \int \left( 1 + c_{1,2}(x_1,x_2) \right) d\pi_{1,2}(x_1,x_2)\\
    &= |(\mathcal{P}_1)_\# \pi_{1,2}| + |(\mathcal{P}_2)_\# \pi_{1,2}| - \int \left( 1 + c_{1,2}(x_1,x_2) \right) d\pi_{1,2}(x_1,x_2)\\
    &= 2|\pi_{1,2}| - \int \left( 1 + c_{1,2}(x_1,x_2) \right) d\pi_{1,2}(x_1,x_2)\\
    &= - \int \left(c_{1,2}(x_1,x_2) - 1\right) d\pi_{1,2}(x_1,x_2).
\end{align*}
Here 
\begin{equation}\label{eq: c_1,2}
    c_{1,2}(x_1, x_2) := \inf_{x'} c(x_1, x') + c(x_2, x')
\end{equation}
This shows that \eqref{eq:multimarginal_decomposed} is equivalent to
\begin{equation}\label{eq: partial OT form of adversarial training}
    \inf_{\pi_{1,2}}\int \left(c_{1,2}(x_1,x_2) - 1\right) d\pi_{1,2} \text{ s.t. } (\mathcal{P}_i)_\# \pi_{1,2} \leq \mu_i \text{ for $i=1,2$,}
\end{equation}
where the constraints should be understood as for any Borel measurable set $A$,
\[
    ((\mathcal{P}_i)_\# \pi_{1,2}) (A) \leq \mu_i(A).
\]

Problem \eqref{eq: partial OT form of adversarial training} is an instance of optimal partial transport, introduced by \citet{CM2010Annalsmath} and subsequently studied extensively in the optimal transport literature; see, for example, \cite{Figalli2010ARMA, FreePOT_Indrei2013, MultiPOT2015, DynamicPOT2016, POT_lagrangian2018}. Introducing a Lagrange multiplier $\alpha \in \mathbb{R}$, which acts as a threshold for the transport cost, we can reformulate the optimal partial transport problem as the minimization of the following functional:
\begin{equation} \label{def: partial transport}
    \mathcal{V}(\alpha) :=\min_{\pi\in\Pi_\leq(\mu_1, \mu_2)}  \int (c(x,x') - \alpha) d\pi(x,x')
\end{equation}       
where
\[
    \Pi_\leq(\mu_1, \mu_2):= \{ \pi : (\mathcal{P}_i)_\# \pi \leq \mu_i \text{ for } i=1,2 \}.
\]
If $\pi_\alpha$ is a unique minimizer for \eqref{def: partial transport}, letting $m(\alpha):= |\pi_\alpha|$, one can deduce that $m(\alpha) = - \frac{\partial \mathcal{V}(\alpha)}{\partial \alpha}$, which increases continuously from $0$ to $\min \{|\mu_1|, |\mu_2| \}$. Also, as in standard optimal transport, the dual of \eqref{def: partial transport} is studied because it provides rich information about the primal problem, which is given as follows:
\begin{equation}\label{def: dual of partial transport}
    \sup_{\substack{ \psi \oplus \phi \leq c - \alpha,   \\ \psi, \phi \leq 0}} \int\psi d\mu_1 + \int \phi d\mu_2.    
\end{equation}
We refer the readers to \cite{CM2010Annalsmath} for duality and uniqueness for the partial transport plan.

\begin{remark}
The equivalence between \eqref{def: distributional model} and \eqref{eq: partial OT form of adversarial training} is implicitly studied in \cite{jakwang_MOT}. \citet{CM2010Annalsmath} show that an optimal partial transport problem is equivalent to a standard optimal transport problem by adjoining an isolated point to balance the total masses of $\mu_1$ and $\mu_2$ and extending the cost function appropriately. A version of this reduction is discussed and used in \cite[Theorem 6]{jakwang_MOT}.
\end{remark}

There is an alternative expression of \eqref{def: partial transport} studied in \cite{CM2010Annalsmath, Figalli2010ARMA, DynamicPOT2016}:
\begin{equation}\label{def: partial transport_equiv}
    \mathcal{U}(m) :=\min_{\pi\in\Pi_\leq(\mu_1, \mu_2;m)}  \int c(x,x')  d\pi(x,x')
\end{equation}       
where
\[
    \Pi_\leq(\mu_1, \mu_2;m):= \{ \pi : |\pi| = m, (\mathcal{P}_i)_\# \pi \leq \mu_i \text{ for } i=1,2 \}.
\]
For each $\alpha \in \mathbb{R}$, if \eqref{def: partial transport} has a unique minimizer, then there is a unique $m(\alpha)=|\pi_\alpha|$, the total amount of partially transported mass from $\mu_1$ to $\mu_2$ given $\alpha$, such that \eqref{def: partial transport} and \eqref{def: partial transport_equiv} share the same solution, and
\begin{equation}\label{eq: equivalence of partial transport}
    \mathcal{V}(\alpha) + \alpha m(\alpha)= \mathcal{U}(m(\alpha)).
\end{equation}

The regularity theory for optimal partial transport with quadratic cost, including uniqueness results for both primal and dual optimizers, has been studied through the Monge-Ampère double obstacle problem in \cite{CM2010Annalsmath} with the quadratic-cost setting. Uniqueness holds under the assumption that $\mu_1$ and $\mu_2$ are supported on bounded, strictly convex domains separated by a hyperplane; see \cite[Corollary 6.4]{CM2010Annalsmath}. \citet{Figalli2010ARMA} relaxes the separation assumption and, among other results, shows that the active region is path-connected; see \cite[Corollary 4.12]{Figalli2010ARMA}. For further results on the regularity of solutions to optimal partial transport problems, we refer to \cite{FreePOT_Indrei2013, DynamicPOT2016}.

The equivalence among \eqref{def: distributional model}, \eqref{eq:multimarginal_decomposed} and \eqref{def: partial transport} (as well as \eqref{def: partial transport_equiv}) plays an important role in establishing the uniqueness of the optimal potential (up to additive constants in the degenerate case). To the best of our knowledge, uniqueness of optimal potentials for optimal partial transport with general cost functions does not follow directly from the existing literature.

\section{Proofs}\label{sec:proofs}

\subsection{Stability of optimal potentials}\label{subsec: stability of optimal potentials}

In this subsection, optimal potentials are chosen in the normalized
$c$-concave form
\[
    g_i=f_i^c,
    \quad 0\leq f_i\leq1,
\]
which is possible by \Cref{thm:learner_part}.

The following lemma is similar to \cite[Theorem 3.3]{gangbo1996geometry}, which states the local Lipschitz property of a $c$-concave function when $c$ satisfies \Cref{assumption: strict convexity of cost function}. It was recently leveraged in \cite[Lemma 2.3]{CLT_general_transport2024} as one of the most important ingredients for the CLT of optimal transport. The boundedness of the potentials allows us to extend local Lipschitz continuity to a global Lipschitz bound.

\begin{lemma}
\label{lem: uniform regularity of potentials}
Assume that $c$ satisfies
\Cref{assumption: strict convexity of cost function}.
Then there exists a constant $L_c<\infty$, depending only on $c$,
such that, for every Borel measurable function
$f:\mathbb R^p\to[0,1]$, its $c$-transform
\[
    f^c(x)
    :=
    \inf_{y\in\mathbb R^p}
    \{f(y)+c(x,y)\}
\]
is $L_c$-Lipschitz.
\end{lemma}

\begin{proof}
Write $c(x,y)=h(x-y)$. By the superlinearity of $h$, there exists
$R>0$ such that
\[
    \{z\in\mathbb R^p:h(z)\leq2\}
    \subset \overline B_R(0).
\]
Since $h$ is finite and convex on $\mathbb R^p$, it is locally
Lipschitz. Hence there exists $L_c<\infty$, depending only on $h$,
such that
\[
    |h(z)-h(z')|
    \leq
    L_c|z-z'|
    \quad
    \text{for all }
    z,z'\in\overline B_{R+1}(0).
\]

Since $f,h\geq0$ and $h(0)=0$,
\[
    0\leq f^c(x)\leq f(x)\leq1
    \quad
    \text{for every }x\in\mathbb R^p.
\]
Fix $x,x'\in\mathbb R^p$ with $|x-x'|\leq1$. For
$\varepsilon\in(0,1)$, choose $y_\varepsilon$ such that
\[
    f(y_\varepsilon)+h(x-y_\varepsilon)
    \leq
    f^c(x)+\varepsilon.
\]
Then
\[
    h(x-y_\varepsilon)
    \leq
    f^c(x)+\varepsilon
    <2,
\]
and therefore $|x-y_\varepsilon|\leq R$.
In particular, $x-y_\varepsilon,\,
    x'-y_\varepsilon
    \in\overline B_{R+1}(0)$.
Thus,
\begin{align*}
    f^c(x')
    &\leq
    f(y_\varepsilon)+h(x'-y_\varepsilon)\\
    &\leq
    f(y_\varepsilon)+h(x-y_\varepsilon)
    +L_c|x-x'|\\
    &\leq
    f^c(x)+\varepsilon+L_c|x-x'|.
\end{align*}
Letting $\varepsilon\downarrow0$ gives
\[
    f^c(x')-f^c(x)
    \leq
    L_c|x-x'|.
\]
Interchanging $x$ and $x'$ yields
\[
    |f^c(x')-f^c(x)|
    \leq
    L_c|x-x'|
\]
whenever $|x-x'|\leq1$.

For arbitrary $x,x'\in\mathbb R^p$, subdividing the line segment
joining $x$ and $x'$ into finitely many segments of length at most one
and applying the preceding estimate on each segment gives
\[
    |f^c(x')-f^c(x)|
    \leq
    L_c|x-x'|,
\]
and we conclude.
\end{proof}

\begin{proposition}[Stability of optimal potentials]
\label{prop: stability of optimal potentials}
Assume that $c$ satisfies
\Cref{assumption: strict convexity of cost function}.
Let $\eta^n,\eta\in\mathcal P(\mathcal X\times\mathcal Y)$ satisfy
\[
    \eta^n\rightharpoonup\eta
    \quad\text{as }n\to\infty.
\]
For each $n$, let
\[
    g^n=(g_i^n)_{i\in\mathcal Y}
\]
be an optimal potential for \eqref{eq:mot_decomposed_dual} with input
$\eta^n$, chosen so that, for every $i\in\mathcal Y$,
\[
    g_i^n=(f_i^n)^c
\]
for some Borel measurable function
$f_i^n:\mathbb R^p\to[0,1]$. Then every subsequence of $\{g^n\}$ admits a further subsequence
converging locally uniformly on $\mathbb R^p$ to an optimal potential
for \eqref{eq:mot_decomposed_dual} with input $\eta$.

In particular, if the optimal potential for the problem with input
$\eta$ is unique, say $g=(g_i)_{i\in\mathcal Y}$, then
\[
    g_i^n\longrightarrow g_i
    \quad
    \text{locally uniformly on }\mathbb R^p
\]
for every $i\in\mathcal Y$.
\end{proposition}

\begin{proof}
By \Cref{lem: uniform regularity of potentials}, the functions
$\{g_i^n\}_n$ are uniformly bounded and equi-Lipschitz on
$\mathbb R^p$. Hence, by the Arzelà--Ascoli theorem and a diagonal
argument, every subsequence of $\{g^n\}$ admits a further subsequence,
still denoted by $\{g^n\}$, such that
\[
    g_i^n\longrightarrow\bar g_i
    \quad
    \text{locally uniformly on }\mathbb R^p
\]
for every $i\in\mathcal Y$. Moreover,
$0\leq\bar g_i\leq1$ and $\bar g_i$ is $L_c$-Lipschitz.

The dual constraints are preserved under this convergence. Indeed, for
every $A\in S_{\mathcal Y}$ and
$x_A\in(\mathbb R^p)^A$,
\[
    \sum_{i\in A}g_i^n(x_i)
    \leq
    1+c_A(x_A).
\]
Passing to the limit gives
\[
    \sum_{i\in A}\bar g_i(x_i)
    \leq
    1+c_A(x_A).
\]
Thus $\bar g=(\bar g_i)_{i\in\mathcal Y}$ is feasible.

We next prove that $\bar g$ is optimal. Since $\mathcal Y$ is finite
and discrete, the weak convergence $\eta^n\rightharpoonup\eta$
implies
\[
    \eta_i^n\rightharpoonup\eta_i
    \quad
    \text{for every }i\in\mathcal Y.
\]
Let
$\widetilde g=(\widetilde g_i)_{i\in\mathcal Y}$ be any feasible
potential. Since $\sum_i\eta_i^n(1)=1$, the optimality of $g^n$ gives
\[
    \sum_{i\in\mathcal Y}\eta_i^n(g_i^n)
    \geq
    \sum_{i\in\mathcal Y}\eta_i^n(\widetilde g_i).
\]

We claim that
\[
    \eta_i^n(g_i^n)
    \longrightarrow
    \eta_i(\bar g_i)
\]
for every $i\in\mathcal Y$. Since
$\eta_i^n\rightharpoonup\eta_i$, the family
$\{\eta_i^n\}_n$ is tight. Given $\varepsilon>0$, choose a compact set
$K\subset\mathbb R^p$ such that
\[
    \sup_n\eta_i^n(K^c)<\varepsilon.
\]
Then, using $0\leq g_i^n,\bar g_i\leq1$,
\begin{align*}
    \left|
        \eta_i^n(g_i^n)-\eta_i(\bar g_i)
    \right|
    &\leq
    \sup_{x\in K}|g_i^n(x)-\bar g_i(x)|
    +\eta_i^n(K^c) \\
    &\quad
    +
    \left|
        \eta_i^n(\bar g_i)-\eta_i(\bar g_i)
    \right|.
\end{align*}
The first term converges to zero by local uniform convergence, while
the last term converges to zero because
$\bar g_i\in C_b(\mathbb R^p)$. Since $\varepsilon>0$ is arbitrary,
the claim follows.

Since each $\widetilde g_i\in C_b(\mathbb R^p)$, weak convergence also
gives
\[
    \eta_i^n(\widetilde g_i)
    \longrightarrow
    \eta_i(\widetilde g_i).
\]
Passing to the limit in the optimality inequality, we obtain
\[
    \sum_{i\in\mathcal Y}\eta_i(\bar g_i)
    \geq
    \sum_{i\in\mathcal Y}\eta_i(\widetilde g_i).
\]
Since $\widetilde g$ was arbitrary, $\bar g$ is an optimal potential
for the problem with input $\eta$.

Finally, suppose that the optimal potential for the limiting problem is
unique, say $g$. Then every subsequence of $\{g^n\}$ admits a further
subsequence converging locally uniformly to $g$. This implies that
the full sequence converges locally uniformly to $g$.
\end{proof}

\subsection{CLT of the empirical adversarial training risk}
In this subsection, we adapt an Efron–Stein argument to establish a CLT centered at the expected empirical adversarial risk. In contrast to \cite{CLT_general_transport2024}, the potentials in our setting are bounded, which directly yields uniform integrability. A distinctive feature is that samples are drawn from $\mu$, which can be viewed as a mixture of the $\mu_i$’s; consequently, the samples associated with different $\mu_i$’s are dependent. This dependence is reflected in the covariance structure of the Gaussian limit.

\begin{theorem}[CLT of empirical adversarial risk]\label{thm: CLT empirical}
Assume that $c$ satisfies \Cref{assumption: strict convexity of cost function}, and there is a unique optimal $g \in \mathcal{G}$ for \eqref{eq:mot_decomposed_dual} with input $\mu$. Then,
\[
    \sqrt{n}(\mathscr{R}(\mu^n)- \mathbb{E}\mathscr{R}(\mu^n) ) \overset{d}{\longrightarrow} \mathcal{N} \left( 0, \mu(\bm{g}^2) - (\mu(\bm{g}) )^2 \right)
\]
where
\[
    \mu(\bm{g}^2) - (\mu(\bm{g}) )^2 = \sum_{i \in \mathcal{Y}} \mu_i \left( g_i^2 \right) - \sum_{i, j \in \mathcal{Y}} \mu_i \left( g_i \right) \mu_j \left( g_j \right).
\]

\begin{proof}
Let
\[
    R_n:= \left(1-\sum_{i \in \mathcal{Y}}  \mu_i^n (g_i) \right) - \mathscr{R}(\mu^n)
\]
where $(g_i)_{i \in \mathcal{Y}}$ is a unique optimal potential tuple for \eqref{eq:mot_decomposed_dual} with input $\mu$. To derive the CLT with the centering of the expectation of empirical adversarial risk, it suffices to show that
\[
    n \text{Var} (R_n) \longrightarrow 0.
\]
Recall \eqref{eq: extension of g_i}:
\[
    \bm{g}(x,y) = \sum_{i \in \mathcal{Y}} g_i(x) \mathds{1}_{y=i}.
\]
Observe that one can rewrite 
\[
    \sum_{i \in \mathcal{Y}}  \mu_i^n(g_i) = \int_{\mathcal{X} \times \mathcal{Y}} \bm{g} d\mu^n.
\]

Let $Z_k=(X_k, Y_k) \overset{i.i.d.}{\sim} \mu$. Recall $\mu_i^n= \frac{1}{n} \sum_{k=1}^{n} \delta_{X_k} \delta_{Y_k=i}$. Let $Z_1'$ be an independent copy of $Z_1$. Let
\[
    \tilde{\mu}^n:= \sum_{i \in \mathcal{Y}} \tilde{\mu}_i^n, \quad \tilde{\mu}_i^n:= \mu_i^n - \frac{1}{n}\delta_{X_1} \delta_{Y_1=k} + \frac{1}{n} \delta_{X_1'} \delta_{Y_1'=k}.
\]
and
\[
    R_n' := \left(1 - \int_{\mathcal{X} \times \mathcal{Y}} \bm{g} d\tilde{\mu}^n \right) - \mathscr{R}(\tilde{\mu}^n).
\]
Applying the Efron–Stein inequality directly with exchangeability, it follows that
\[
    n \text{Var} (R_n) \leq n^2 \mathbb{E} \left( R_n - R_n'\right)_+^2.
\]
Hence, the goal reduces to showing
\begin{equation}\label{eq: consequence of efron-stein}
    n^2 \mathbb{E} \left( R_n - R_n'\right)_+^2 \longrightarrow 0.
\end{equation}
In order to show \eqref{eq: consequence of efron-stein}, it suffices to prove that $n \left(R_n - R_n' \right)_+ \to 0$ almost surely, and $n^2 \left(R_n - R_n' \right)_+^2$ is uniformly integrable.

Let $g^n$ be optimal for \eqref{eq:mot_decomposed_dual} with input $\mu^n$. We use $\bm{g}^n$ as defined in \eqref{eq: extension of g_i} with $g^n$. By \Cref{thm:learner_part},
\[
    R_n' \geq \int \bm{g}^n d \tilde{\mu}^n  - \int \bm{g} d \tilde{\mu}^n ,
\]
which implies for $Y_1=i, Y_1'=j$,
\begin{align*}
    R_n - R_n' &\leq \int \bm{g}^n d (\mu^n - \tilde{\mu}^n) + \int \bm{g} d( \tilde{\mu}^n - \mu^n)\\
    &=\frac{1}{n} \left( g_i^n(X_1) - g_i(X_1) - g_j^n(X_1') + g_j(X_1') \right).
\end{align*}
\Cref{prop: stability of optimal potentials} implies that $g^n \to g$ pointwise almost surely, from which it holds that
$n \left(R_n - R_n' \right)_+ \longrightarrow 0$ almost surely.

It remains to show $n^2 \left(R_n - R_n' \right)_+^2$ is uniformly integrable. Since $0 \leq g_i, g_i^n \leq 1$ everywhere, it holds that
\[
    n \left(R_n - R_n' \right)_+ \leq 4.
\]
Thus, $n^2 \left(R_n - R_n' \right)_+^2 \leq 16$, which implies uniform integrability. This completes the claim that the limiting distribution is Gaussian. A straightforward calculation shows that the variance of the limiting distribution is
\[
  \int_{\mathcal{X} \times \mathcal{Y}} \bm{g}^2 d\mu - \left( \int_{\mathcal{X} \times \mathcal{Y}} \bm{g} d\mu \right)^2 = \sum_{i \in \mathcal{Y}} \int g_i^2 d\mu_i - \sum_{i, j \in \mathcal{Y}} \int g_i d\mu_i \int g_j d\mu_j.  
\]
\end{proof}
\end{theorem}

\subsection{CLT of the smoothed adversarial training risk}

In this subsection, we derive the limiting distribution of the smoothed adversarial training risk. The argument relies on two main ingredients. The first is the asymptotic tightness of the relevant empirical process. This is established using the extension of $g=(g_i)_{i \in \mathcal{Y}}$ to $\bm{g}$ defined in \eqref{eq: extension of g_i}, together with the fact that the class of convolutions $g * \chi$, where $g$ ranges over bounded measurable functions and $\chi$ is smooth, is Donsker under a sufficiently strong tail-decay condition on the underlying measure.

The second ingredient is the stability of the optimal dual potentials. The uniqueness result established in \Cref{subsection: uniqueness}, together with the stability of optimal potentials, yields convergence of the empirical optimizer to the population optimizer after smoothing.
Combined with the asymptotic equicontinuity of the empirical process, an optimality argument identifies the first-order fluctuation with the empirical process evaluated at the population optimizer. The desired limiting distribution then follows from the classical central limit theorem. Related limiting distribution results for smoothed optimal transport
functionals have recently been studied in
\cite{limit_pWasserstein2024, Goldfeld_Kato_Rioux_Sadhu2024}.

Let us introduce basic notions regarding empirical process theory based on \cite{MR4628026, Sen_introEmpirical2022}. Given a (real-valued function) space $\mathcal{S}$, let $\ell^{\infty}(\mathcal{S})$ be the set of all bounded real-valued functionals on $\mathcal{S}$ with the supremum norm defined as
\[
    || Z ||_\mathcal{S} := \sup_{f \in \mathcal{S}} |Z(f)|.
\]
Equipped with $|| \cdot ||_\mathcal{S}$, $\ell^{\infty}(\mathcal{S})$ is a Banach space.

\begin{definition}
Let $(\mathcal{X}, \mathcal{S}, \mu)$ be a metric measure space and $\mathcal{S}$ be a class of real-valued measurable functions defined over $\mathcal{X}$. Consider the empirical process 
\[
    \{\mathbb{G}_{n,\mu}f:= \sqrt{n}(\mu^n(f) - \mu (f)): f \in \mathcal{S}\}
\]
where $\mu^n$ is an empirical distribution of $n$ i.i.d. samples $X_i \sim \mu$. A class $\mathcal{S}$ is called a $\mu$-Donsker class if
\begin{equation}
    \mathbb{G}_{n,\mu} \longrightarrow \mathbb{G}_\mu
\end{equation}
in distribution to a tight Borel measurable element in the space $\ell^{\infty}(\mathcal{S})$. 
\end{definition}

We need to extend $g_i$'s defined on $\mathcal{X}$ to a function on $\mathcal{X} \times \mathcal{Y}$ to obtain a correct empirical process associated to $\mu$. Given $\mathcal{S}$ a class of measurable functions on $\mathcal{X}$, define a function class on $\mathcal{X} \times \mathcal{Y}$ by
\[
    \widetilde{\mathcal{S}}:=\{\widetilde f:f\in\mathcal{S}\}, \quad \widetilde f(x,y):=f(x).
\]
Recall \eqref{eq: extension of g_i}. Given $\mathcal{S}$ a class of measurable functions on $\mathcal{X}$, define a function class on $\mathcal{X} \times \mathcal{Y}$ by
\[
    \mathcal H:=\bigcup_{k=1}^{K}\mathcal{H}_i, \quad \mathcal{H}_i:=\left\{ h_{i,f}(x,y):=\mathds{1}_{\{y=i\}} \widetilde{f}(x,y): \widetilde{f} \in \mathcal{S} \right\}.
\]
Then, $\mathcal{H}$ is $\mu$-Donsker if $\mathcal{S}$ is $\mu_X$-Donsker.

\begin{proposition}\label{prop: donsker}
Assume that $\mathcal{S}$ is $\mu_X$-Donsker, and $\sup_{f\in\mathcal{S}}|\mu_X (f)|<\infty$. Then $\mathcal H$ is $\mu$-Donsker.
\end{proposition}

\begin{proof}
For each $f\in\mathcal{S}$, $\mu(\widetilde f) = \mu_X(f)$, hence 
\begin{align*}
    \sqrt n(\mu^n-\mu)(\widetilde f)
    =\frac1{\sqrt n}\sum_{k=1}^{n}\left(\widetilde f(X_k, Y_k)-\mu(\widetilde f)\right)=\frac1{\sqrt n}\sum_{k=1}^{n}\left(f(X_k)-\mu_X(f)\right).
\end{align*}
The variables $X_k$ are i.i.d.\ drawn by $\mu_X$.
Therefore, the $\mu$-empirical process indexed by
$\widetilde{\mathcal{S}}$ has exactly the same distribution as the
$\mu_X$-empirical process indexed by $\mathcal{S}$. Since $\mathcal{S}$ is $\mu_X$-Donsker, it follows that $\widetilde{\mathcal{S}}$ is $\mu$-Donsker. Furthermore,
\[
    \sup_{\widetilde f\in\widetilde{\mathcal{S}}}
|\mu(\widetilde f)|=\sup_{f\in\mathcal{S}}|\mu_X(f)|<\infty.
\]
Since $\mathcal{H}_i = \mathds{1}_{y=i} \widetilde{\mathcal{S}}:= \{\mathds{1}_{y=i} \widetilde{f}: f \in \widetilde{f} \}$, \Cref{lem: donkser_lemma} yields $\mathcal{H}_i$ is $\mu$-Donsker. Since this holds for every $i \in \mathcal{Y}$, all $\mathcal H_i$'s are $\mu$-Donsker.

For each $i$, define the component empirical process 
\[
    \mathbb{G}_{n, i}(f):= \sqrt n(\mu^n-\mu)(h_{i,f}).
\]
Since $\mathcal{H}_i$ is $\mu$-Donsker, $\{\mathbb{G}_{n, i}:n\geq 1\}$ is asymptotically tight in $\ell^\infty(\mathcal{S})$. Since $|\mathcal{Y}|=K<\infty$, the vector process
\[
    (\mathbb{G}_{n,i})_{i \in \mathcal{Y}}
\]
is asymptotically tight in $\left( \ell^\infty(\mathcal{S}) \right)^K$ equipped with the supremum norm.

Asymptotic tightness and finite-dimensional convergence imply
\[
    (\mathbb{G}_{n,i})_{i \in \mathcal{Y}} \overset{d}{\longrightarrow} (\mathbb{G}_{i})_{i \in \mathcal{Y}} \quad \text{in $\left( \ell^\infty(\mathcal{S}) \right)^K$. }
\]
Equivalently,
\[
    \sqrt n(\mu^n-\mu) \overset{d}{\longrightarrow} \mathbb{G}_\mu\quad \text{in }\ell^\infty(\mathcal H).
\]
Therefore, the conclusion follows.
\end{proof}

\begin{remark}\label{rmk: covariance strucutre of Gaussian process}
Let us explain the covariance structure of the limiting Gaussian process. For $f,g\in\mathcal{S}$ and $i,j\in\mathcal{Y}$,
\begin{align*}
\mu(h_{i,f}h_{j,g})=\mu\left(
\mathds{1}_{\{y=i\}}\mathds{1}_{\{y=j\}}
f(x)g(x)\right)=\mathds{1}_{\{i=j\}}\mu_i(fg).
\end{align*}
Consequently, 
\[
    \text{Cov}\left( \mathbb{G}_\mu(h_{i,f}),\mathbb{G}_\mu(h_{j,g})
\right) = \text{Cov}\left( \mathbb{G}_i(f),\mathbb{G}_j(g)
\right)=
\mathds{1}_{\{i=j\}}\mu_i(fg)-
\mu_i(f)\mu_j(g).
\]
In particular, for $i\neq j$,
\[
    \text{Cov}\left(\mathbb{G}_i(f),\mathbb{G}_j(g)
\right)=-\mu_i(f)\mu_j(g).
\]    
\end{remark}

\begin{lemma}\label{lem: donkser_lemma}
Let $\mu$ be a probability measure, and let $\mathcal{S}$ be a
$\mu$-Donsker class satisfying $\sup_{g\in\mathcal{S}}|\mu(g)|<\infty.$ If $a$ is a fixed measurable function satisfying $\norm{a}_{\infty}\leq 1$, then
\[
    a\mathcal{S}:= \{ag:g\in\mathcal{S}\}
\]
is $\mu$-Donsker.
\end{lemma}

\begin{proof}
In the proof, a constant $L$ would change line by line but still be universal, meaning that it does not depend on the given parameters, which is the convention in empirical process theory.

Let $\mathbb{G}_{n,\mu}$ denote the empirical process associated with $\mu$. We will show that if $\mathbb{G}_{n,\mu}$ has the limit in $\ell^\infty(\mathcal{S})$, so as in $\ell^\infty(a\mathcal{S})$.

Define the canonical semimetric of the $\mu$-Brownian bridge by
\[
    \rho^2_\mu(g,g'):=\text{Var}_\mu(g-g')
\]
and the $L^2(\mu)$ semimetric by
\[
    d^2_\mu(g,g'):= \| g-g' \|^2_{L^2(\mu)}.
\]
These satisfy
\[
    d^2_\mu(g,g')=\rho^2_\mu(g,g') + \left(\mu (g) -\mu(g') \right)^2.
\]
Since $\mathcal{S}$ is $\mu$-Donsker,
$(\mathcal{S},\rho_\mu)$ is totally bounded. In addition, $\sup_{g\in\mathcal{S}}|\mu(g)|<\infty$ implies that $(\mathcal{S},d_\mu)$ is also totally bounded.

Let us pick a finite
$\rho_\mu$-net for $\mathcal{S}$, which is possible due to total boundedness. Since $\{\mu(g):g\in\mathcal{S}\}$ is a bounded subset of $\mathbb R$, it admits $\rho_\mu$-net with finite size. Refining the $\rho_\mu$-net according to these finitely
many mean intervals produces a finite $d_\mu$-net. For $\delta>0$, define
\[
    \mathcal U_\delta:=\left\{g-g': g,g'\in\mathcal{S},\; d_\mu(g,g')<\delta \right\}.
\]
Consider the Rademacher process
\[
    \mathcal{R}_n (u) := \frac{1}{\sqrt{n}}\sum_{i=1}^n \epsilon_i u(X_i)
\]
where $\epsilon_i$'s are i.i.d.\ Rademacher variable. Standard symmetrization gives
\[
    \mathbb{E}^* \sup_{u\in\mathcal U_\delta}
\left|
    \mathbb{G}_{n,\mu}(au) \right| \leq 2 \mathbb{E}^* \sup_{u\in\mathcal U_\delta}
\left|
    \mathcal{R}_n (au) \right|.
\]
Since $\| a \|_\infty \leq 1$, Ledoux-Talagrand contraction inequality \cite[Theorem 4.12]{ledoux1991probability} provides
\[
    \mathbb{E}^* \sup_{u\in\mathcal U_\delta}
\left|
    \mathcal{R}_n (au) \right| \leq L \mathbb{E}^* \sup_{u\in\mathcal U_\delta}
\left|
    \mathcal{R}_n (u) \right|
\]
for some universal constant $L$. Thus,
\[
    \mathbb{E}^* \sup_{u\in\mathcal U_\delta}
\left|
    \mathbb{G}_{n,\mu}(au) \right| \leq L \mathbb{E}^* \sup_{u\in\mathcal U_\delta}
\left|
    \mathcal{R}_n (u) \right|.
\]

Rewrite $\mathcal{R}_n (u)$ as
\[
    \mathcal{R}_n (u) := \frac{1}{\sqrt{n}}\sum_{i=1}^n \epsilon_i (u(X_i) - \mu(u)) + \mu(u)\frac{1}{\sqrt{n}}\sum_{i=1}^n \epsilon_i=: \mathcal{R}^o_n (u) + \mu(u) S_n.
\]
Thus, 
\[
    \sup_{u\in\mathcal U_\delta}
\left|
    \mathcal{R}_n (u) \right| \leq \sup_{u\in\mathcal U_\delta}
\left|
    \mathcal{R}^o_n (u) \right| + \sup_{u\in\mathcal U_\delta}
\left|
    \mu(u) S_n \right|.
\]
For $u\in\mathcal U_\delta$, using the fact that $\mu$ is a probability measure and Jensen's inequality gives
\[
    |\mu(u)| \leq \mu (|u|) \leq \|u \|_{L^2(\mu)} < \delta.
\]
On the other hand, since $\mathbb{E}|S_n|^2=1$,
\[
    \mathbb{E} |S_n| \leq (\mathbb{E} S_n^2)^{\frac{1}{2}} =1.
\]
Combining the above calculations, we have
\[
    \mathbb{E}^* \sup_{u\in\mathcal U_\delta}
\left|
    \mathbb{G}_{n,\mu}(au) \right| \leq L \mathbb{E}^* \sup_{u\in\mathcal U_\delta}
\left|
    \mathcal{R}^o_n (u) \right| + \delta.
\]

Let $X'_i$'s be i.i.d.\ copies of the samples. Since
\[
    u(X_i) - \mu(u) = \mathbb{E}[ u(X_i) - u(X_i')| X_i],
\]
Jensen's inequality implies
\[
    \mathbb{E}^* \sup_{u\in\mathcal U_\delta}
\left|
    \mathcal{R}^o_n (u) \right| \leq \mathbb{E}^* \sup_{u\in\mathcal U_\delta}
\left| \frac{1}{\sqrt{n}} \sum_{i=1}^n \epsilon_i( u(X_i) - u(X_i')) \right|.
\]
Using exchangeability and symmetry by multiplication by $-1$, it follows that
\[
     \mathbb{E}^*_{X, X', \epsilon} \sup_{u\in\mathcal U_\delta}
\left| \frac{1}{\sqrt{n}} \sum_{i=1}^n \epsilon_i( u(X_i) - u(X_i')) \right| = \mathbb{E}^*_{X, X'} \sup_{u\in\mathcal U_\delta}
\left| \frac{1}{\sqrt{n}} \sum_{i=1}^n (u(X_i) - u(X_i')) \right|. 
\]
Using
\[
    u(X_i) - u(X_i') = u(X_i) - \mu(u) + \mu(u) - u(X_i')
\]
and the triangle inequality, it turns out that
\[
    \mathbb{E}^* \sup_{u\in\mathcal U_\delta}
\left|
    \mathcal{R}^o_n (u) \right| \leq L \mathbb{E}^* \sup_{u\in\mathcal U_\delta}
\left|
    \mathbb{G}_{n,\mu}(u) \right|.
\]
To summarize, there is a universal constant $L$ such that
\begin{equation}\label{eq: bound G_au by G_u}
    \mathbb{E}^* \sup_{u\in\mathcal U_\delta}
\left|
    \mathbb{G}_{n,\mu}(au) \right| \leq  L \left( \mathbb{E}^* \sup_{u\in\mathcal U_\delta}
\left|
    \mathbb{G}_{n,\mu}(u) \right| + \delta \right).
\end{equation}

Since $\mathbb{G}_{n,\mu}$ is tight in $\ell^\infty(\mathcal{S})$, and $\rho_\mu\leq d_\mu$, one has 
\[
    \lim_{\delta\downarrow 0}
    \limsup_{n\to\infty} \mathbb{P}^*\left(\sup_{u\in\mathcal U_\delta}
\left| \mathbb{G}_{n,\mu}(u) \right|>\eta\right)=0.
\]
In particular,
\[
    \lim_{\delta\downarrow 0}
    \limsup_{n\to\infty} \mathbb{E}^* \sup_{u\in\mathcal U_\delta}
\left| \mathbb{G}_{n,\mu}(u) \right| + \delta = 0.
\]
Combined with \eqref{eq: bound G_au by G_u}, the above yields
\begin{equation}\label{eq: aymptotic control of G(au)}
    \lim_{\delta\downarrow 0}
    \limsup_{n\to\infty} \mathbb{E}^* \sup_{u\in\mathcal U_\delta}
\left| \mathbb{G}_{n,\mu}(au) \right| = 0.
\end{equation}
By Markov inequality, for any $\eta > 0$, \eqref{eq: aymptotic control of G(au)} implies
\[
    \mathbb{P}^*\left(\sup_{u\in\mathcal U_\delta}
\left|
    \mathbb{G}_{n,\mu}(au) \right|>\eta\right) \leq \frac{1}{\eta} \mathbb{E}^* \sup_{u\in\mathcal U_\delta}
\left| \mathbb{G}_{n,\mu}(au) \right|.
\]
Therefore, for every $\eta>0$,
\[
\lim_{\delta\downarrow 0}
    \limsup_{n\to\infty} \mathbb{P}^*\left(\sup_{u\in\mathcal U_\delta}
\left|
    \mathbb{G}_{n,\mu}(au) \right|>\eta\right)=0.
\]
This shows that $a \mathcal{S}$ is $\mu$-Donsker.
\end{proof}

\begin{lemma}\label{lem: convolution makes donsker}
Assume that $\mu_X$ satisfies
\Cref{assumption: measure for empirical process}
and that $\chi$ satisfies
\Cref{assumption: smoothing kernel}. Let
\[
    \mathcal S_\chi
    :=
    \left\{
        f*\chi :
        f:\mathbb R^p\to[0,1]
        \text{ is Borel measurable}
    \right\}.
\]
Then $\mathcal S_\chi$ is $\mu_X$-Donsker.
\end{lemma}

\begin{proof}
By \Cref{assumption: smoothing kernel} and the standard convolution estimate,
\[
    \sup_{f\in\mathcal S_\chi}
    \|\partial^\beta f\|_\infty
    \leq
    \|\partial^\beta\kappa\|_{L^1(\mathbb R^p)},
    \quad |\beta|\leq m.
\]
Thus, there exists
$M_\chi<\infty$ such that
\[
    \|f\|_{C^m(\mathcal X_j)}\leq M_\chi
\]
for every $f\in\mathcal S_\chi$ and every $j$.

By \Cref{assumption: measure for empirical process},
\[
    \sum_j M_\chi\mu_X(\mathcal X_j)^{1/2}<\infty.
\]
Since $m>p/2$, the conclusion follows from
\cite[Example~2.10.28]{MR4628026}.
\end{proof}

Following \eqref{eq: extension of g_i}, define
\begin{align}
\label{eq:gchi2}
    \bm g_\chi^n(x,y)
    :=
    \sum_{i\in\mathcal Y}
    (g_i^n*\chi)(x)\mathds 1_{\{y=i\}},
    \quad
    \bm g_\chi(x,y)
    :=
    \sum_{i\in\mathcal Y}
    (g_i*\chi)(x)\mathds 1_{\{y=i\}}.
\end{align}

\begin{lemma}\label{lem: L2 convergence smoothed potentials}
Let $g^n=(g_i^n)_{i\in\mathcal Y}$ and
$g=(g_i)_{i\in\mathcal Y}$ be optimal potentials for
\eqref{eq:mot_decomposed_dual} with input measures
$\mu^{n,\chi}$ and $\mu^\chi$, respectively.
Assume that $g$ is the unique optimal potential for the problem with input $\mu^\chi$. Let $\bm g_\chi^n$ and $\bm g_\chi$ be given in \eqref{eq:gchi2}. Then, 
\[
    \|\bm g_\chi^n-\bm g_\chi\|_{L^2(\mu)}
    \longrightarrow 0 \quad \text{almost surely}.
\]

\end{lemma}

\begin{proof}
Let
\[
    \Omega_0
    :=
    \left\{
        \omega:
        \mu^n(\omega)\rightharpoonup\mu
    \right\}.
\]
Then $\mathbb P(\Omega_0)=1$. Fix $\omega\in\Omega_0$. Since
convolution with the fixed probability measure $\chi$ preserves weak
convergence,
\[
    \mu^{n,\chi}(\omega)\rightharpoonup\mu^\chi.
\]
Hence, by \Cref{prop: stability of optimal potentials} and the uniqueness of
$g$,
\[
    g_i^n\longrightarrow g_i
    \quad\text{locally uniformly on }\mathbb R^p
\]
almost surely for every $i\in\mathcal Y$.

Since $0\leq g_i^n,g_i\leq1$, the dominated convergence theorem gives
\[
    (g_i^n*\chi)(x)\longrightarrow(g_i*\chi)(x)
    \quad\text{for every }x\in\mathbb R^p
\]
almost surely. Moreover,
\[
    0\leq g_i^n*\chi,\;g_i*\chi\leq1.
\]
Applying the dominated convergence theorem once more, we obtain
\[
    \int_{\mathbb R^p}
    \left|
        (g_i^n*\chi)(x)-(g_i*\chi)(x)
    \right|^2
    d\mu_i(x)
    \longrightarrow0
\]
almost surely. Therefore,
\[
\begin{aligned}
    \|\bm g_\chi^n-\bm g_\chi\|_{L^2(\mu)}^2
    =
    \sum_{i\in\mathcal Y}
    \int_{\mathbb R^p}
    \left|
        (g_i^n*\chi)(x)-(g_i*\chi)(x)
    \right|^2
    d\mu_i(x) \longrightarrow0.
\end{aligned}
\]
Since this holds for every $\omega\in\Omega_0$, the conclusion follows.
\end{proof}

Now we are ready to prove the central limit theorem of smoothed adversarial training risk of \eqref{def: distributional model} centered at the population value.

\begin{theorem}
\label{thm: limit distribution for smoothed}
Assume that $c$ satisfies
\Cref{assumption: strict convexity of cost function},
that $\mu_X$ satisfies
\Cref{assumption: measure for empirical process},
and that $\chi$ satisfies
\Cref{assumption: smoothing kernel}.
Assume further that the dual problem
\eqref{eq:mot_decomposed_dual} with input $\mu^\chi$
admits a unique optimal potential
\[
    g=(g_i)_{i\in\mathcal Y},
    \quad
    0\leq g_i\leq1.
\]
Then
\[
    \sqrt n
    \left(
        \mathscr R(\mu^{n,\chi})
        -
        \mathscr R(\mu^\chi)
    \right)
    \overset{d}{\longrightarrow}
    \mathcal N(0,\sigma_\chi^2),
\]
where
\begin{align*}
    \sigma_\chi^2
    :=
    \mu\left((\bm g_\chi)^2\right)
    -
    \mu(\bm g_\chi)^2
    =
    \sum_{i\in\mathcal Y}
    \mu_i\left((g_i*\chi)^2\right)
    -
    \sum_{i,j\in\mathcal Y}
    \mu_i(g_i*\chi)\mu_j(g_j*\chi),
\end{align*}
and $\bm g_\chi$ is given in \eqref{eq:gchi2}.
\end{theorem}

\begin{proof}
Recall the set of feasible dual potentials $\mathcal{G}$ defined in \eqref{def: dual potential space}. Since $\chi$ is symmetric, Fubini's theorem gives $(\eta_i*\chi)(g_i)
    =
    \eta_i(g_i*\chi)$
for every finite measure $\eta_i$ on $\mathbb R^p$. Consequently, the
dual representation yields
\begin{equation}\label{eq: smoothed dual representation}
    \mathscr R(\mu^\chi)
    =
    1-\sup_{g \in \mathcal{G}}\mu(\bm g_\chi),
    \quad
    \mathscr R(\mu^{n,\chi})
    =
    1-\sup_{g \in \mathcal{G}}\mu^n(\bm g_\chi).
\end{equation}

We first establish the Donsker property of the relevant function class. By \Cref{lem: convolution makes donsker},
\[
    \mathcal{S}_\chi:=\left\{
        g*\chi:
        h:\mathbb R^p\to[0,1]
        \text{ is Borel measurable}
    \right\}
\]
is $\mu_X$-Donsker. Moreover, since $0\leq g*\chi\leq1$ for every
$g*\chi\in\mathcal \mathcal{S}_\chi$,
\[
    \sup_{g*\chi \in\mathcal G_\chi}|\mu_X(g*\chi)|\leq1.
\]
Therefore, \Cref{prop: donsker}, applied with
$\mathcal{S}=\mathcal{S}_\chi$, implies that the corresponding finite sum class
\[
    \left\{
        \sum_{i\in\mathcal Y}
        \mathds 1_{\{y=i\}}g_i(x):
        g_i\in \mathcal{S}_\chi
    \right\}
\]
is $\mu$-Donsker. Consequently, its subclass
\[
    \mathcal G_\chi
    :=
    \left\{
        \bm g_\chi:g \in \mathcal{G}
    \right\}
\]
is $\mu$-Donsker.

Recalling \eqref{eq:gchi2}, let
\[
    \mathbb G_n(\bm{h})
    :=
    \sqrt n\,(\mu^n-\mu)(\bm{h}).
\]
By the optimality of $g^n$ for the empirical problem,
\[
    \mu^n(\bm g_\chi^n)
    \geq
    \mu^n(\bm g_\chi),
\]
and hence
\[
    \mathbb G_n(\bm g_\chi)
    \leq
    \sqrt n(\mu^n(\bm g_\chi^n)-\mu(\bm g_\chi)).
\]
On the other hand, by the optimality of $g$ for the population problem,
\[
    \mu(\bm g_\chi)
    \geq
    \mu(\bm g_\chi^n),
\]
and therefore
\[
    \sqrt n(\mu^n(\bm g_\chi^n)- \mu(\bm g_\chi))
    \leq
    \mathbb G_n(\bm g_\chi^n).
\]
Thus,
\begin{equation}\label{eq: smoothed CLT sandwich}
    \mathbb G_n(\bm g_\chi)
    \leq
    \sqrt n(\mu^n(\bm g_\chi^n)- \mu(\bm g_\chi))
    \leq
    \mathbb G_n(\bm g_\chi^n).
\end{equation}

Since $\mathcal G_\chi$ is $\mu$-Donsker, the empirical process
$\mathbb G_n$ is asymptotically uniformly equicontinuous with respect to its canonical semimetric
\[
    \rho_\mu(f_1,f_2)
    :=
    \sqrt{\operatorname{Var}_\mu(f_1-f_2)}.
\]
Since
\[
\begin{aligned}
    \rho_\mu^2(f_1,f_2) \leq \|f_1-f_2\|_{L^2(\mu)}^2.
\end{aligned}
\]
we have by \Cref{lem: L2 convergence smoothed potentials}, 
\[
    \rho_\mu(\bm g_\chi^n,\bm g_\chi)
    \longrightarrow 0
    \quad \text{almost surely}.
\]
In particular,
\begin{equation}\label{eq: convergence in probability}
    \rho_\mu(\bm g_\chi^n,\bm g_\chi)
    \longrightarrow 0
    \quad \text{in probability}.
\end{equation}

Since the sequence
$\{\mathbb G_n\}_{n\geq 1}$ is asymptotically uniformly
$\rho_\mu$-equicontinuous on $\mathcal G_\chi$, it follows that $\epsilon>0$,
\[
    \lim_{\delta\downarrow 0}
    \limsup_{n\to\infty}
    \mathbb P^*\left(
        \sup_{\substack{\bm{f},\bm{h}\in\mathcal{H}_\chi\\
                        \rho_\mu(\bm{f}, \bm{h})<\delta}}
        \left|\mathbb G_n(\bm{f})-\mathbb G_n(\bm{h})\right|
        >\epsilon
    \right)
    =0,
\]
from which, for any fix $\epsilon,\eta>0$, one can choose $\delta>0$ such that
\[
    \limsup_{n\to\infty}
    \mathbb P^*\left(
        \sup_{\substack{\bm{f},\bm{h}\in\mathcal{H}_\chi\\
                        \rho_\mu(\bm{f}, \bm{h})<\delta}}
        \left|\mathbb G_n(\bm{f})-\mathbb G_n(\bm{h})\right|
        >\epsilon
    \right)
    <\eta.
\]
Observe that
\[
\begin{aligned}
    &\mathbb P^*\left(
        \left|
            \mathbb G_n(\bm g_\chi^n)
            -
            \mathbb G_n(\bm g_\chi)
        \right|
        >\epsilon
    \right)\\
    &\leq
    \mathbb P\left(
        \rho_\mu(\bm g_\chi^n,\bm g_\chi)\geq\delta
    \right)  +
    \mathbb P^*\left(
        \sup_{\substack{f,h\in\mathcal{H}_\chi\\
                        \rho_\mu(f,h)<\delta}}
        \left|\mathbb G_n(f)-\mathbb G_n(h)\right|
        >\epsilon
    \right).
\end{aligned}
\]
The first term converges to zero by \eqref{eq: convergence in probability}, and the second
term is less than $\eta$ since $\mathcal{H}_\chi$ is $\mu$-Donsker. Hence
\[
    \limsup_{n\to\infty}
    \mathbb P^*\left(
        \left|
            \mathbb G_n(\bm g_\chi^n)
            -
            \mathbb G_n(\bm g_\chi)
        \right|
        >\epsilon
    \right)
    \leq \eta.
\]
Since $\eta>0$ is arbitrary, it follows that
\[
    \mathbb G_n(\bm g_\chi^n)
    -
    \mathbb G_n(\bm g_\chi)
    \longrightarrow 0
    \quad \text{in probability}.
\]
Combining this with
\eqref{eq: smoothed CLT sandwich} gives
\[
    \sqrt n(\mu^n(\bm g_\chi^n)- \mu(\bm g_\chi))
    -
    \mathbb G_n(\bm g_\chi)
    \longrightarrow 0 \quad \text{in probability}.
\]

By \eqref{eq: smoothed dual representation},
\[
    \mathscr R(\mu^{n,\chi})-\mathscr R(\mu^\chi)
    = \mu(\bm g_\chi) - \mu^n(\bm g_\chi^n).
\]
Therefore,
\[
    \sqrt n
    \left(
        \mathscr R(\mu^{n,\chi})
        -
        \mathscr R(\mu^\chi)
    \right)
    +
    \mathbb G_n(\bm g_\chi)
    \longrightarrow0 \quad \text{in probability}.
\]

Finally, since $\bm g_\chi$ is bounded, the classical central limit
theorem yields
\[
    \mathbb G_n(\bm g_\chi)
    \overset{d}{\longrightarrow}
    \mathcal N(0,\sigma_\chi^2),
\]
where
\[
    \sigma_\chi^2
    =
    \mu\bigl((\bm g_\chi)^2\bigr)
    -
    \mu(\bm g_\chi)^2.
\]
Since the centered Gaussian distribution is symmetric, it follows from Slutsky's theorem that
\[
    \sqrt n
    \left(
        \mathscr R(\mu^{n,\chi})
        -
        \mathscr R(\mu^\chi)
    \right)
    \overset{d}{\longrightarrow}
    \mathcal N(0,\sigma_\chi^2),
\]
which completes the proof.
\end{proof}

\subsection{Uniqueness}\label{subsection: uniqueness}

In this subsection, we study the uniqueness property of the adversarial training model. All uniqueness statements are understood to hold $\mu$-almost everywhere.

First, we show the uniqueness of the optimal adversarial attack. Note that an optimal adversarial attack exists uniquely in the general multiclass case.

\begin{lemma}\label{lem: uniqueness of adversarial attack}
Assume that $c$ satisfies \Cref{assumption: strict convexity of cost function}, and $\mu_i$'s are absolutely continuous measures. Then, there is a unique tuple of optimal adversarial attacks for \eqref{def: distributional model}.

\begin{proof}
Recalling \eqref{def: distributional model}, one can rewrite it as
\begin{equation}
    1 - \inf_{\nu_1, \dots, \nu_i} \left\{ \left( \sup_{f \in \mathcal{F}} \sum_{i=1}^K \nu_i(f_i)  \right) + C(\mu_i, \nu_i) \right\}.
\end{equation}
 Since $\nu \mapsto \sum \nu_i(f_i)$ is linear, its supremum is convex with respect to $\nu_i$'s. Also, $\nu_i \mapsto C(\mu_i, \nu_i)$ is convex. Since optimal adversarial attacks exist \cite[Proposition 7]{jakwang_MOT}, it suffices to show the strict convexity of the functional with respect to $\nu_i$'s.

Without loss of generality, let us focus on
\[
    \nu \mapsto C(\mu, \nu).
\]
Fix $\nu_0 \neq \nu_1$ with the same mass as $\mu$. Since $\mu$ is absolutely continuous, by \cite[Theorem 2.4]{CLT_general_transport2024}, for each $i=0,1$, there is a unique optimal transport map $T^i$ such that $(T^i)_\# \mu = \nu_i$. For each $t \in [0,1]$, let $\nu_t = (1-t) \nu_0 + t \nu_1$. By convexity,
\[
    C(\mu, \nu_t) \leq (1-t)C(\mu, \nu_0) + t C(\mu, \nu_1).
\]
Suppose the equality holds for a contradiction. Then, it is easy to check that the linear interpolation of couplings
\[
    (1-t) (\mathrm{Id}, T^0)_\# \mu + t (\mathrm{Id}, T^1)_\# \mu 
\]
is optimal for $C(\mu, \nu_t)$. Since $\mu$ is absolutely continuous and $c$ satisfies \Cref{assumption: strict convexity of cost function}, there is a unique optimal coupling between $\mu$ and $\nu_t$ which is induced by a transport map. This implies that $T^0=T^1$, which contradicts to $\nu_0 \neq \nu_1$. Therefore, a strict inequality holds, which shows the strict convexity. This completes the proof.
\end{proof}
\end{lemma}

Although optimal adversarial attack is unique for multiclass classification generally, its associated MOT problem may have multiple solutions. For \eqref{def: stratified system of barycenter problem}, despite the uniqueness of $\lambda^*$, which is uniquely determined by a unique optimal attack, corresponding optimal partition $\{ \mu_A \}$ might not be unique. However, restricted to the binary setting, \eqref{def: stratified system of barycenter problem} also admits a unique solution.

\begin{lemma}\label{lem: unique partition}
Consider the binary setting. Assume that $c$ satisfies \Cref{assumption: strict convexity of cost function}, and $\mu_i$'s are absolutely continuous measures. Then, there is a unique optimal solution for \eqref{def: stratified system of barycenter problem} up to sets of measure zero.

\begin{proof}
We need to recall the proof of \cite[Proposition 11]{jakwang_MOT}. For fixed unique optimal attacks $\nu^*_i$'s, let $\lambda^*$ be the smallest (finite and positive) measure covering them, which is uniquely determined. Since $\nu^*_i \ll \lambda^*$, define
\[
    v_i(x) = \frac{d\nu^*_i}{d \lambda^*},
\]
which is unique for almost every $x \in \spt(\lambda^*)$. Note that by the definition of $\lambda^*$,
\[
    \max_{i} v_i(x) = 1, \quad 0 \leq v_i \leq 1.
\]

Define a collection of non-negative scalars $\{r_A(x)\}$ satisfying the following:
\begin{enumerate}
    \item $\sum_{A \in S_\mathcal{Y}} r_A(x) = \max_i v_i(x) = 1$.
    \item $\sum_{A \in S_\mathcal{Y}(i)} r_A(x) = v_i(x)$ for all $i \in \mathcal{Y}$.
\end{enumerate}
Suppose such $\{r_A\}$ is obtained. Accordingly, define
\[
    d\lambda^*_A = r_A(x) d \lambda^*
\]
for each $A \in S_\mathcal{Y}$. It is not hard to see that
\[
    \sum_{A \in S_\mathcal{Y}(i)} d\lambda^*_A(x) =  \sum_{A \in S_\mathcal{Y}(i)} r_A(x) d \lambda^* = v_i(x) d \lambda^* = d \nu^*_i(x).
\]
Let $\gamma^*_i = (\mathrm{Id}, T^i)_\# (\mu_i)$ be a unique optimal coupling for $C(\mu_i,\nu^*_i)$. By the disintegration formula,
\[
    \gamma^*_i(dx, d\tilde{x}) = \gamma^*_i(dx |\tilde{x}) \nu^*_i(d\tilde{x}).
\]
For each $A\in S_\mathcal{Y}$, let
\[
    \gamma^*_{i,A}(dx, d\tilde{x}):= \gamma^*_i(dx |\tilde{x}) r_A(\tilde{x}) d \lambda^*(d \tilde{x}),
\]
and set $\mu_{i, A}$ to be the first marginal ($x$-marginal) of $\gamma^*_{i,A}$. Then, $\{\lambda^*_A, \mu_{i,A}\}$ is optimal for \eqref{def: stratified system of barycenter problem}. Therefore, for the uniqueness of the optimal partition, it suffices to show that such $\{r_A\}$ is a unique collection for $\lambda^*$-almost every $x$.

For the binary setting $(K=2)$, note that $\{r_1, r_2, r_{1,2}\}$ should satisfy
\[
    r_1 + r_2 + r_{1,2} = \max\{v_1, v_2\}, \quad r_1 + r_{1,2}=v_1, \quad  r_2 + r_{1,2} = v_2.
\]
Rearranging them appropriately, one can see that
\[
    r_{1,2} = v_1 + v_2 -\max\{v_1, v_2\}, \quad r_1 = \max\{v_1, v_2\} - v_2, \quad r_2 = \max\{v_1, v_2\} - v_1,
\]
which shows that $\{r_1, r_2, r_{1,2}\}$ is determined uniquely. This completes the proof.
\end{proof}
\end{lemma}

\begin{remark}
It should be emphasized that even if $\nu^*_i$'s are unique for the multiclass case, the above argument is not applicable for \eqref{def: stratified system of barycenter problem} with $K \geq 3$ due to the fact that $\{r_A\}$ is not uniquely determined for $K \geq 3$. However, it does not exclude the uniqueness of the optimal partition for the multiclass case. We leave it as an open problem.
\end{remark}

Recall \eqref{eq:mot_decomposed_dual} and \eqref{def: dual of partial transport} with a cost function $c_{1,2}$ defined in \eqref{eq: c_1,2}. The following lemma shows that they are equivalent.

\begin{lemma}\label{lemma: dual partial tranport = dual of adv}
For each $\alpha \in \mathbb{R}$,
\begin{align}\label{eq: equivalence POT and dual of adv}
    \eqref{def: dual of partial transport} =   - \inf_{\substack{ g_1 \oplus g_2 \leq c_{1,2} + \alpha,   \\ 0 \leq g_1, g_2 \leq \alpha \vee 0}} \left\{ \alpha - \left(\int g_1 d\mu_1 + \int g_2 d\mu_2 \right) \right\}. 
\end{align}
In particular, for \eqref{def: dual of partial transport}, it is sufficient to consider $(-\alpha) \wedge 0 \leq (\psi, \phi) \leq 0$

\begin{proof}
For simplicity, choose $\alpha \geq 0$. The right-hand side of \eqref{eq: equivalence POT and dual of adv} turns out to be equal to
\[
    \sup_{\substack{ g_1 \oplus g_2 \leq c_{1,2} + \alpha,   \\  g_1, g_2 \leq \alpha}} \left\{ \int (g_1 - \alpha) d\mu_1 + \int (g_2 - \alpha) d\mu_2 \right\}.
\]
Note that the condition $g_1, g_2 \leq \alpha$ comes from the constraint of \eqref{eq:mot_decomposed_dual}: taking $A =\{i \}$, it follows that
\[
    g_i(x) \leq c_{i}(x) + \alpha = \min_{x'} c(x,x')+ \alpha = 0
\]
since $c$ vanishes on the diagonal. Letting $\psi := g_1 - \alpha$ and $\phi:=g_2 - \alpha$, the above becomes
\[
    \sup_{\substack{ \psi \oplus \phi \leq c_{1,2} - \alpha,   \\ \psi, \phi \leq 0}} \left\{ \int \psi d\mu_1 + \int \phi d\mu_2 \right\}.
\]
Since it is sufficient to restrict $g_i \geq 0$, the conclusion follows.
\end{proof}
\end{lemma}

\begin{theorem}\label{thm: uniquness of potential}
Consider the binary setting. Assume that $c$ satisfies \Cref{assumption: strict convexity of cost function}, and $\mu_1$ and $\mu_2$ are absolutely continuous measures with connected supports. Then, the following holds.
\begin{enumerate}
    \item[(i)] If $\mu_{1,1}=\mu_{2,2}=0$ (degenerate case), optimal potential for \eqref{eq:mot_decomposed_dual} is unique up to additive constants, i.e., if $g$ and $g'$ are optimal, there is a constant $a \in \mathbb{R}$ such that
\[
    g_1' = g_1 + a, \quad g_2'=g_2 - a.
\]
    \item[(ii)] If either $\mu_{1,1}$ or $\mu_{2,2}$ does not vanish (non-degenerate case), then optimal $g_1$ from $\mu_1$ to $\mu_2$ is unique.
\end{enumerate}

\begin{proof}
For the binary setting, there is a unique optimal partition for \eqref{def: stratified system of barycenter problem} by \Cref{lem: unique partition}. Let $\{\mu_{i,A}\}$ be the corresponding unique partition. Since it satisfies the marginal constraints,
\[
    \sum_{A \in S_\mathcal{Y}(i)} \mu_{i, A}= \mu_i,
\]
each of $\mu_{i, A}$'s is also absolutely continuous. We claim that there is also a unique solution for \eqref{eq:multimarginal_decomposed}. First, trivially,
\[
    \pi^*_1= \mu_{1,1}, \quad \pi^*_2 = \mu_{2,2},
\]
which are uniquely determined up to sets of measure zero with respect to $\mu_1$ and $\mu_2$, respectively. Let $\pi^*_{1,2}$ be an optimal coupling between $\mu_{1,\{1,2\}}$ and $\mu_{2, \{1,2\}}$, which is well-defined since $| \mu_{1,\{1,2\}}| = | \mu_{2, \{1,2\}} | = |\lambda^*_{1,2}|$; see the proof of \Cref{lem: unique partition}. Recall the transport problem associated to $\pi^*_{1,2}$:
\begin{equation}\label{eq: pi_{1,2}}
    \int (1 + c_{1,2}(x_1, x_2)) d\pi_{1,2}(x_1, x_2) \text{ s.t. } \mathcal{P}_i\pi_{1,2} = \mu_{i, \{1,2\}},
\end{equation}
where
\[
    c_{1,2}(x_1, x_2) := \inf_{x'} c(x_1, x') + c(x_2, x').
\]
By \Cref{lem: min of c_1 + c_2}, $c_{1,2}$ also satisfies \Cref{assumption: strict convexity of cost function}. Then, \cite[Theorem 2.4]{CLT_general_transport2024} implies $\pi^*_{1,2}$ is unique. In particular, due to the equivalence of \eqref{eq:multimarginal_decomposed}, \eqref{def: partial transport} and \eqref{def: partial transport_equiv} with $\alpha=1$, $\pi^*_{1,2}$ is a unique solution for \eqref{def: partial transport_equiv} with $m^*:=m(1)=|\pi^*_{1,2}|$.

Let
\[
    \overline{\mu}_1 := \mu_1 + \mu_{2,2}, \quad \overline{\mu}_2:= \mu_2 +\mu_{1,1},
\]
and consider the standard transport problem:
\begin{equation}\label{eq: auxilary OT}
    \inf_{\pi \in \Pi(\overline{\mu}_1, \overline{\mu}_2)} \int c_{1,2}(x_1, x_2)d \pi.
\end{equation}
Note that 
\begin{align*}
    |\overline{\mu}_1| = |\mu_{1,1}| + |\mu_{1, \{1,2\}}|+ |\mu_{2,2}| = |\mu_{1,1}| + |\mu_{2, \{1,2\}}|+ |\mu_{2,2}| = |\mu_{1,1}| + |\mu_{2}| = |\overline{\mu}_2|,
\end{align*}
which guarantees the existence of solutions for \eqref{eq: auxilary OT}. Following the standard optimal transport argument, the dual of \eqref{eq: auxilary OT} is given as
\begin{equation}\label{eq: dual of auxilary OT}
    \sup_{\psi + \phi \leq c_{1,2}} \int \psi d\overline{\mu}_1 + \int \phi d\overline{\mu}_2.
\end{equation}

Given $\pi \in \Pi(\overline{\mu}_1, \overline{\mu}_2)$, the disintegration formula reads to
\[
    \pi(dx_1, dx_2) = \pi(dx_2|x_1) \overline{\mu}_1(dx_1).
\]
Define 
\[
    \pi'(dx_1, dx_2):= \pi(dx_2|x_1) \mu_1(dx_1).
\]
Let $\mu'_1$ and $\mu'_2$ be its first and second marginals. Clearly, $\mu'_1 = \mu_1$ and $\mu'_2 \leq \overline{\mu}_2$. In particular, considering $\mu'_2 \wedge \mu_2$, it follows that
\begin{align*}
    \int d (\mu'_2 \wedge \mu'_2) &= \int d (\mu'_2 \wedge (\mu_2 +\mu_{1,1})) - \int d (\mu'_2 \wedge \mu_{1,1} )\\
    &\geq \int d \mu'_2 - \int d \mu_{1,1}\\
    &= |\pi'| - |\mu_{1,1}|\\
    &= |\mu_1| - |\mu_{1,1}| = m^*.
\end{align*}
Now, define
\[
    \tilde{\pi}(dx_1, dx_2):= \pi(dx_1|x_2) (\mu'_2 \wedge \mu_2)(dx_2).
\]
Then, $\tilde{\pi}$ has the total mass $m^*$, hence is feasible for \eqref{def: partial transport_equiv} with $m=m^*$. Furthermore, it is straightforward that $\tilde{\pi} \leq \pi$.

Now, considering \eqref{def: partial transport_equiv} with $m=m^*$, since $|\tilde{\pi}|=m^*$ it holds that
\[
    \int c_{1,2} d\tilde{\pi} \geq \mathcal{U}(m^*).
\]
On the other hand, since $\tilde{\pi} \leq \pi$, 
\[
    \int c_{1,2} d\pi \geq \int c_{1,2} d\tilde{\pi}.
\]
Since it holds for arbitrary $\pi \in \Pi(\overline{\mu}_1, \overline{\mu}_2)$, it leads to
\[
    \inf_{\pi \in \Pi(\overline{\mu}_1, \overline{\mu}_2)} \int c_{1,2} d\pi \geq \mathcal{U}(m^*).
\]
Let $\overline{\pi} := \pi^*_{1,2} +(\mathrm{Id}, \mathrm{Id})_\#(\mu_{1,1} + \mu_{2,2})$. It is not hard to see that
\begin{equation}\label{eq: auxilary OT = partial OT equiv}
    \int c_{1,2} d\overline{\pi} = \mathcal{U}(m^*),
\end{equation}
which shows that $\overline{\pi}$ is a solution for \eqref{eq: auxilary OT}. Furthermore, since $c_{1,2}$ satisfies \Cref{assumption: strict convexity of cost function} by \Cref{lem: min of c_1 + c_2}, \cite[Theorem 2.4]{CLT_general_transport2024} implies that $\overline{\pi}$ is the unique solution for \eqref{eq: auxilary OT}.

Since $\overline{\mu}_1$ is absolutely continuous, there is a unique optimal transport map $T^*(x) = x - \nabla c_{1,2}^* (\nabla \psi (x))$ from $\overline{\mu}_1$ to $\overline{\mu}_2$ which generates a unique solution $\overline{\pi}$ where $c_{1,2}^*$ is a convex conjugate of $c_{1,2}$, and $\psi$ is an optimal potential for \eqref{eq: dual of auxilary OT}. Since the supports of $\mu_1$ is connected and $T^*$ is unique, by \Cref{lem: uniqueness of potenial}, if $\psi_1$ and $\psi_1'$ are both optimal for \eqref{eq: dual of auxilary OT}, there is $a \in \mathbb{R}$ such that $\psi_1' = \psi_1 + a$ on the support of $\mu_1$.

Let $(\tilde{\psi}, \tilde{\phi})$ be a feasible pair for \eqref{def: dual of partial transport}. %Set $\tilde{\psi}:=-1$ on $\mathcal{X} \setminus \spt(\mu_1)$ and $\tilde{\phi}:=-1$ on $\mathcal{X} \setminus \spt(\mu_2)$. This extension still satisfies the constraint of \eqref{def: dual of partial transport} with preserving the optimality. Therefore, from now on this extension of $(\tilde{\psi}, \tilde{\phi})$ is concerned.
Define
\[
    \psi:= \tilde{\psi} + \frac{1}{2}, \quad 
    \phi:= \tilde{\phi} + \frac{1}{2}.
\]
It is direct from the constraint of \eqref{def: dual of partial transport} that for any $(x_1, x_2)$,
\[
    \psi(x_1)+ \phi(x_2) \leq \tilde{\psi}(x_1) + \tilde{\phi}(x_2) + 1 \leq c_{1,2}(x_1, x_2),
\]
which implies that $(\psi, \phi)$ is feasible for \eqref{eq: dual of auxilary OT}. If $(\tilde{\psi}, \tilde{\phi})$ is optimal, then
\begin{align*}
    \int \psi d\overline{\mu}_1 + \int \phi d\overline{\mu}_2 
    &= \int \tilde{\psi} d\overline{\mu}_1 + \int \tilde{\phi} d\overline{\mu}_2 + \frac{1}{2}\left(|\mu_1| + |\mu_2| + |\mu_{2,2}| + |\mu_{1,1}| \right)\\ 
    &= \int \tilde{\psi} d\mu_1 + \int \tilde{\phi} d\mu_2 + \frac{1}{2}\left(|\mu_1| + |\mu_2| - |\mu_{2,2}| -|\mu_{1,1}| \right)\\
    &= \mathcal{V}(1) + m^* \\
    &= \mathcal{U}(m(1))\\
    &= \int c_{1,2} d\overline{\pi}
\end{align*}
where the last and the second last equalities follow from \eqref{eq: auxilary OT = partial OT equiv} and \eqref{eq: equivalence of partial transport}, respectively. It turns out that $(\psi, \phi)$ is optimal for \eqref{eq: dual of auxilary OT}. Since $\varphi$ is unique up to additive constants on the support of $\mu_1$ by the previous argument, one can deduce that $\tilde{\psi}$ is also unique up to additive constants. %If $\mu_{1,1}=\mu_{2,2}=0$ (degenerate case), then the problem reduces to the standard optimal transport, and optimal potentials are unique up to additive constants as usual. 

Now, let us pay attention to the case that either $\mu_{1,1}$ or $\mu_{2,2}$ does not vanish. Returning to \eqref{eq:mot_decomposed_dual}, \Cref{lemma: dual partial tranport = dual of adv} shows that $g_1:= \tilde\psi + 1$ and $g_2:= \tilde\phi + 1$ is optimal for \eqref{eq:mot_decomposed_dual}. Recall $\{\pi^*_{1}, \pi^*_2, \pi^*_{1,2} \}$ is a unique optimal solution for \eqref{eq:multimarginal_decomposed}. Then, for any optimal $(g_1, g_2)$, it should satisfy
\[
    g_i(x) = 1 \text{ on $\spt(\pi^*_i)$ for $i=1,2$}, \quad g_1(x_1) + g_2(x_2) = 1 + c_{1,2}(x_1, x_2) \text{ on $\spt(\pi^*_{1,2})$}.
\]
If $g_1, g_1'$ are both optimal, then $\nabla g_1 = \nabla g_1'$ on the support of $\mu_1=\pi^*_1 + (\mathrm{P}_1)_\#\pi^*_{1,2}$. Since $g_1=g_1'=1$ on the support of $\pi^*_1=\mu_{1,1}$, which does not vanish, and $\spt(\mu_1)$ is connected, it follows that $g_1=g_1'$ on $\spt(\mu_1)$. The same argument implies $g_2=g_2'$. This completes the proof.
\end{proof}
\end{theorem}

Under the uniqueness (up to additive constants) of the optimal potential, as $\mu^n \to \mu$, a sequence of optimal potentials corresponding to $\mu^n$ has an appropriate convergent subsequence whose limit is the (up to additive constants) unique optimal potential.

\begin{theorem}\label{thm: pointwise limit of g}
Consider the binary setting. Assume that $c$ satisfies \Cref{assumption: strict convexity of cost function}, $\mu^n \to \mu$ weakly, and $\mu_1$ and $\mu_2$ are absolutely continuous and have connected supports. Let $0 \leq g^n, g \leq 1$ be optimal potentials for \eqref{eq:mot_decomposed_dual} with input $\mu^n$ and $\mu$, respectively. Then
\begin{enumerate}
    \item If $\mu_{1,1}=\mu_{2,2}=0$ (degenerate case), there exists a sequence $\{(a_n,a
    _n) \in \mathbb{R}^2\}$ such that $(g^n_1-a_n, g^2 -a_n')$ converges to $(g_1,g_2)$ pointwise.
    \item If either $\mu_{1,1}$ or $\mu_{2,2}$ does not vanish (non-degenerate case), then $g^n$ converges to $g$ pointwise.
\end{enumerate}

\begin{proof}
For (1), this is the case of the standard optimal transport: see \cite[Theorem 3.4]{CLT_general_transport2024}.

For (2), \Cref{prop: stability of optimal potentials} implies that $\{(g^n_1, g^n_2)\}$ has a convergent subsequence of which limit is $(g'_1, g'_2)$ which is optimal for \eqref{eq:mot_decomposed_dual} with input $\mu$. Since $(g_1, g_2)$ is a unique optimal potential for \eqref{eq:mot_decomposed_dual} with input $\mu$ by \Cref{thm: uniquness of potential}, $(g'_1, g'_2)=(g_1, g_2)$, which implies that the conclusion.
\end{proof}
\end{theorem}

A next lemma shows that $c_A$ satisfies \Cref{assumption: strict convexity of cost function} provided that $c$ does.

\begin{lemma}\label{lem: min of c_1 + c_2}
Assume that $c$ satisfies \Cref{assumption: strict convexity of cost function}. Consider  $c_{1,2}(x_1, x_2)$ defined in \eqref{eq: c_1,2}. Then, it satisfies \Cref{assumption: strict convexity of cost function}. Furthermore, 
\[
    c_{1,2}(x_1, x_2) = 2 h \left( \frac{x_1 - x_2}{2} \right).
\]

\begin{proof}
Define
\[
    H(x,y):=\min_{z}\{h(x-z)+h(y-z)\}.
\]
By the coerciveness and strict convexity of $c$, there is a unique minimizer. Fix $x,y$, and set
\[
    a:=\frac{x-y}{2}, \quad b:=\frac{x+y}{2}.
\]
For any $z$, write $v:=b-z$. Then
\[
    x-z=a+v,
    \quad
    y-z=-a+v, 
\]
which implies
\[
    h(y-z)=h(-a+v)=h(a-v).
\]
Therefore
\[
    h(x-z)+h(y-z)
    =
    h(a+v)+h(a-v).
\]
By convexity of $h$,
\[
    \frac12 h(a+v)+\frac12 h(a-v)
    \geq
    h\left(\frac{(a+v)+(a-v)}{2}\right)
    =
    h(a).
\]
Hence
\[
    h(a+v)+h(a-v)\geq 2h(a).
\]
Equality holds when $v=0$, i.e. when
\[
    z=b=\frac{x+y}{2}.
\]
Thus
\[
    H(x,y)=2h(a)
    =
    2h\left(\frac{x-y}{2}\right).
\]
This completes the proof.
\end{proof}
\end{lemma}

The next lemma is one of the key ingredients to derive the central limit theorem of optimal transport with general cost $c$ satisfying \Cref{assumption: strict convexity of cost function} used in \cite{CLT_general_transport2024}. The idea of the proof is Poincaré’s inequality (\cite[Theorem 3.2]{acosta2004optimal}) in a connected domain combined with the local Lipschitzness of $c$-concave functions (\cite[Theorem 3.3]{gangbo1996geometry}).

\begin{lemma}\cite[Corollary 2.7]{CLT_general_transport2024}\label{lem: uniqueness of potenial}
Assume that $c(x,x')$ satisfies \Cref{assumption: strict convexity of cost function}, $\mu$ is absolutely continuous and is supported on an open connected set. If $\psi_1,\psi_2$ are optimal transport potentials from $\mu$ to $\nu$ for the cost $c$, then there exists a constant $a \in \mathbb{R}$ such that $\psi_2(x)= \psi_1(x) + a$ for every $x \in \Omega$.    
\end{lemma}

This subsection is concluded after presenting useful facts regarding convolution.

\begin{lemma}\label{lem:connectedsupport}
For any finite positive measure $\mu$,  $\mathrm{int}(\spt(\mu*\chi))$ has negligible boundary. Assume further that $\spt(\chi)$ is connected. Then, if $\spt(\mu)$ is open connected, $\spt(\mu*\chi)$ also open connected.

\begin{proof}
The first claim is from \cite[Lemma C.16]{Goldfeld_Kato_Rioux_Sadhu2024}. Regarding the second claim, observe that
\[
    \spt(\mu*\chi) = \overline{\spt(\mu) + \spt(\chi)}
\]
where $A+B:=\{ a+b : a \in A, b \in B\}$ and the overline denotes the closure. 
\end{proof}
\end{lemma}

\subsection{Convergence of saddle points}

In this subsection, we consider the convergence of robust classifiers and adversarial attacks.

The stability of optimal dual potentials is proved in \Cref{prop: stability of optimal potentials}. Since $f=g^{\overline{c}}$ recovers the optimal robust classifier by \Cref{thm:learner_part}, it is sufficient to focus on the convergence of optimal adversarial attacks.

Define the set of optimal adversarial attacks for $\mu$ as
\[
    \mathcal{S}(\mu):= \left\{ \nu \in \mathcal{P}(\mathcal{X} \times \mathcal{Y}) : \mathscr{R}(\mu) = \mathscr{R}(\nu;\mu) \right\}.
\]
Note that for any Borel probability measure $\mu$, $\mathcal{S}(\mu)$ is not empty \cite[Proposition 7]{jakwang_MOT}.

\begin{lemma}\label{lem: adversary's lemma}
\quad
\begin{enumerate}
    \item[(i)] $(\mu,\nu) \mapsto \mathscr{R}(\nu; \mu)$ is upper semicontinuous. %; $\nu \mapsto \mathscr{R}(\nu; \mu)$ is concave.
    \item[(ii)] Let $\mu^n \to \mu$ weakly as $n \to \infty$. Then, any sequence $\{ \nu^n \in \mathcal{S}(\mu^n)\}_{n}$ is tight.
    %\item[(iii)] $\mu \mapsto \mathscr{R}(\mu)$ is upper semicontinuous.
\end{enumerate}
\end{lemma}

\begin{proof}
Recalling \eqref{def: distributional model}, we have 
\begin{align*}
    \mathscr{R}(\nu; \mu) &= \inf_{f \in \mathcal{F}} \mathscr{R}(f,\nu; \mu)= 1 +\inf_{f \in \mathcal{F}} \left\{ \sum_{i \in \mathcal{Y}} \int_\mathcal{X} -f_i(x)d\nu_i(x) \right\} - C(\mu, \nu).
\end{align*}

$(i)$ By \cite[Lemma 4.3]{Oldandnew}, $-C(\mu, \nu)$ is upper semicontinuous with respect to $\nu$ and $\mu$. Since $\nu \mapsto \sum_{i \in \mathcal{Y}} \int -f_i d\nu_i$ is a continuous linear functional for $\nu$, the infimum of continuous functionals is upper semicontinuous.

%Regarding concavity, by \cite[Theorem 4.8]{Oldandnew}, $C_\varepsilon$ is convex respect to $\nu$. It is straightforward that $\nu \mapsto \inf_{f \in \mathcal{F}}\sum_{i \in \mathcal{Y}} \int -f_i d\nu_i$ is concave since the infimum of linear functionals is concave.

$(ii)$ For each $n$, let $ \nu^n \in \mathcal{S}(\mu^n)$. Since $\{\mu^n\}_n$ is tight and each $\mu^n$ is a probability measure, for each $\delta > 0$, there is a compact subset $A_{\delta}$ such that for any $i \in \mathcal{Y}$, $\inf \mu_i^n(A_{\delta}) > 1 - \delta$. Define
\begin{align*}
    B_\delta&:= \text{closure of }\left\{x': c(x',x) \leq 1 \text{ for some $x \in A_\delta$} \right\}. 
\end{align*}
For any sequence $(x'_n)_n \subseteq B_\delta$, there is a corresponding sequence $(x_n)_n \subseteq A_\delta$ such that $c(x'_n, x_n) \leq 1$ (by using the axiom of choice). By \Cref{assumption: cost function}, $\{(x'_n, x_n)\}$ is relatively compact; hence, there is a convergent subsequence of $(x'_{n_k})_k$, which shows that $B_\delta$ is compact.

Now, it remains to show that for any $\delta \in (0,1)$ and any $i \in \mathcal{Y}$,
\begin{equation}\label{eq: tightness}
    \inf_{n} \nu_i^n (B_{\delta}) > 1- \delta
\end{equation}
Suppose that it is false. Then, for some $i \in \mathcal{Y}$ some mass of $A_{\delta}$, say $\alpha > 0$, is transported to $B_{\delta}^c$ for infinitely many $n$. By the construction of $B_\delta$, the transport cost of such mass that the adversary should pay is at least $\alpha$, hence the maximum risk that the adversary obtains is $\alpha(1-\frac{1}{K}) - \alpha = -\frac{1}{K} \alpha < 0$ where $\frac{1}{K}$ is obtained by the uniform classification rule, which is the worst for the learner, on the set of the associated transported mass. However, if the adversary does not transport this mass, the minimum risk associated with such mass is $0$, obtained by perfect classification, and there is no transport cost. The risk for the adversary in this case is $0$, which contradicts the optimality of $\nu$. Hence, any positive mass of $A_\delta$ should not be transported to $B_{\delta}^c$. Therefore, \eqref{eq: tightness} holds.
%$(iii)$ Assume that $\mu^n \to \mu$ weakly. For each $n$, let $\nu^n \in \mathcal{S}(\mu^n)$. By $(ii)$, for any sequence $\{ \nu^n \in \mathcal{S}(\mu^n) \}$, there is a convergent subsequence of it. Let $\nu$ be such a subsequent limit. By the upper semicontinuity of $(\mu,\nu) \mapsto \inf_{f \in \mathcal{F}}\mathscr{R}(f, \nu; \mu)$, it follows that
%\[
%    \limsup_{n \to \infty} \mathscr{R}(\mu^n) = \limsup_{n \to \infty} \inf_{f \in \mathcal{F}}\mathscr{R}(f, \nu^n; \mu^n) \leq \inf_{f \in \mathcal{F}}\mathscr{R}(f, \nu; \mu) \leq \mathscr{R}(\mu),
%\]
%which verifies $\mu \mapsto \mathscr{R}(\mu)$ is upper semicontinuous.
\end{proof}

The following shows the convergence of optimal adversarial attacks.

\begin{corollary}\label{cor: continuity of attack}
Assume that $\mu^n \to \mu$ weakly. Let $\nu^n \in \mathcal{S}(\mu^n)$ for each $n$. If 
\begin{equation}\label{eq: continuity of adversarial risk}
    \limsup_{n \to\infty} \mathscr{R}(\mu^n) = \mathscr{R}(\mu),
\end{equation}
then any subsequent limit of $(\nu^n)$ is optimal for $\mu$.

In particular, if $c$ satisfies \Cref{assumption: strict convexity of cost function}, then the conclusion holds.

\begin{remark}\label{rmk: upper semicontinuity}
It is not hard to see that $\mu \mapsto \mathscr{R}(\mu)$ is upper semicontinuous under \Cref{assumption: cost function}. If \eqref{eq: continuity of adversarial risk} holds further, then $\mathscr{R}(\cdot)$ becomes continuous under the weak topology.
\end{remark}

\begin{proof}
Fix $\nu^n \in \mathcal{S}(\mu^n)$ for each $n$. Let $\nu$ be the limit of its subsequence, which is guaranteed to exist by \Cref{lem: adversary's lemma} $(ii)$. Relabeling accordingly, we write $\nu^n \to \nu$ weakly. By \Cref{lem: adversary's lemma} $(i)$, the upper semicontinuity of $(\mu,\nu) \mapsto \mathscr{R}(\nu;\mu)$ reads to
\begin{align*}
    \mathscr{R}(\nu;\mu)= \mathscr{R}(\lim_{n \to \infty} \nu^n; \lim_{n \to \infty} \mu^n) &\geq \limsup_{n \to \infty} \mathscr{R}(\nu^n ;  \mu^n)= \limsup_{n \to \infty} \mathscr{R}(\mu^n) = \mathscr{R}(\mu),
\end{align*}
which implies $\nu \in \mathcal{S}(\mu)$.

In particular, if $c$ satisfies \Cref{assumption: strict convexity of cost function}, \Cref{prop: stability of optimal potentials} implies that \eqref{eq: continuity of adversarial risk} holds. 
\end{proof}    
\end{corollary}

\subsection{Generalization error}
In this subsection, the generalization error of \eqref{def: distributional model} is studied. We show that the generalization error is bounded by the $W_1$ distance between $\mu$ and $\mu^n$.

Recall the dual form of the $1$-Wasserstein distance:
\begin{align*}
    W_1(\mu,\nu) =\sup \left\{ \int \varphi d(\mu -\nu) :  \| \varphi \|_{Lip} \leq 1 \right\}.
\end{align*}

\begin{proposition}\label{prop: generalization error for Lischitz cost}
Assume $c$ satisfies \Cref{assumption: strict convexity of cost function}. Let $f^n$ and $f^*$ be robust classifiers for \eqref{def: distributional model} with inputs $\mu^n$ and $\mu$, respectively. Let
\[
    \Delta(f^n;\mu):=\mathscr{R}(f^n; \mu) - \mathscr{R}(f^*; \mu).
\]
Then, 
\[
   0 \leq \Delta(f^n;\mu) \leq 2 (2 \vee L_c) W_1(\mu, \mu^n)
\]
where $L_c > 0$ is a constant depending on $c$.

\begin{proof}
Recalling \Cref{thm:learner_part}, given $f^n$ and $f^*$, their $c$-transforms are given as $g^n := (f^n)^c$ and $g^*:=(f^*)^c$, respectively. Using the notation \eqref{eq: extension of g_i}, let
\[
    \bm{g}^n(x,y) := \sum_{i \in \mathcal{Y}}g^n_i(x) a_i(x,y), \quad \bm{g}^*(x,y) := \sum_{i \in \mathcal{Y}}g_i(x) a_i(x,y).
\]
\Cref{thm:learner_part} implies that
\begin{align*}
    \Delta(f^n;\mu)&= \sum_{i \in \mathcal{Y}} \int (f^*_i)^c d\mu_i - \sum_{i \in \mathcal{Y}} \int (f^n_i)^c d\mu_i\\
    &= \sum_{i \in \mathcal{Y}} \int g^*_i d\mu_i - \sum_{i \in \mathcal{Y}} \int g^n_i d\mu_i\\
    &= \mu(\bm{g}^*) - \mu(\bm{g}^n).
\end{align*}
Observe that
\[
    \mu(\bm{g}^*) - \mu(\bm{g}^n) =  \mu(\bm{g}^*) - \mu^n(\bm{g}^*) + \mu^n(\bm{g}^*) - \mu^n(\bm{g}^n) + \mu^n(\bm{g}^n) - \mu( \bm{g}^n)
\]    
and $\mu^n(\bm{g}^*) - \mu^n(\bm{g}^n) \leq 0$. These yield that
\begin{align*}
    0 \leq \Delta(f^n;\mu) \leq 2 \sup_{\bm{g}} \left| \mu(\bm{g}) - \mu^n(\bm{g}) \right| = 2 \sup_{\bm{g}} \left| \int \bm{g}(x,y)d ( \mu^n - \mu) \right|.
\end{align*}
By \Cref{lem: uniform regularity of potentials}, $g_i$'s are $L_c$-Lipschitz. Since $\bm{g}$ is also $(2 \vee L_c)$-Lipschitz by \Cref{lem: extension of complete metric space}, it turns out that
\[
    0 \leq \Delta(f^n;\mu) \leq 2 (2 \vee L_c) W_1(\mu, \mu^n).
\]
The conclusion follows.
\end{proof}
\end{proposition}

The following lemma shows that $\bm{g}$ is Lipschitz on $\mathcal{X} \times \mathcal{Y}$ equipped with an appropriate metric naturally arising from a metric on $\mathcal{X}$ if $g_i$'s are Lipschitz on $\mathcal{X}$.

\begin{lemma}\label{lem: extension of complete metric space}
Let $(\mathcal{X},d_\mathcal{X})$ be a complete metric space. On $\mathcal{X} \times \{1, \dots, K\}$, define a metric
\[
    d( (x,y), (x',y'):= d_\mathcal{X}(x,x') + \mathds{1}_{y \neq y'}.
\]
Then, $(\mathcal{X} \times \{1, \dots, K\}, d)$ is a complete metric space.

Furthermore, if $g_i : \mathcal{X} \to \mathbb{R}$ are Lipschitz, then
\[
    \bm{g}(x,y) := \sum_{i=1}^K g_i(x) \mathds{1}_{y=i}
\]
is also Lipschitz over $(\mathcal{X} \times \{1, \dots, K\}, d)$.

\begin{proof}
The first claim is obvious. Regarding the second claim, let $L_i$ be the Lipschitz constant of $g_i$. Picking $(x,y), (x',y')$, it is direct that if $y=y'$, then 
\begin{align*}
    |\bm{g}(x,y) - \bm{g}(x',y')| = |g_y(x) - g_{y}(x')| 
        \leq L_y d_\mathcal{X}(x,x').
\end{align*}
If $y\neq y'$, then 
\begin{align*}
    |\bm{g}(x,y) - \bm{g}(x',y')| 
    %&= g_y(x) - g_{y'}(x)  + g_{y'}(x) - g_{y'}(x')\\
    &\leq \left|\sum_{i=1}^K \left(g_i(x) \mathds{1}_{y=i} - g_i(x) \mathds{1}_{y'=i} \right)\right| + L_{y'} d_\mathcal{X}(x,x')  \\
    %&= |\sum_{i=1}^K g_i(x) (\mathds{1}_{y=i} - \mathds{1}_{y'=i})| + L_{y'} d_\mathcal{X}(x,x')\\
    &\leq 2+ L_{y'} d_\mathcal{X}(x,x')\\
    &\leq (2 \vee \max\{L_i\}) d((x,y), (x',y')),
\end{align*}
which shows the Lipschitz continuity of $\bm{g}$.  %with the Lipschitz constant $\max\{L_1, L_2, \cdots, L_K, K\}$.
\end{proof}
\end{lemma}

One can obtain an upper bound on the sample complexity of the generalization error by combining the recent development of statistical optimal transport (\Cref{thm : upper bound of expected distance} and \Cref{lem:concentration}). Since the proof is straightforward, we omit it.

\begin{corollary}\label{cor: sample complexity of generalization error}
Assume $c$ satisfies \Cref{assumption: strict convexity of cost function}. Let $L_c > 0$ be a constant depending on $c$. 
\begin{enumerate}
    \item[(i)] If $\mu$ has a finite $3$-rd moment, i.e. $M_3(\mu):=||X||_{L^3(\mu)} < \infty$, then
\begin{equation*}
    \begin{aligned}
    0 \leq \mathbb{E}\Delta(f^n;\mu) \leq 20 d (2 \vee L_c) M_3(\mu)  
    \left\{ 
    \begin{array}{ll}
    n^{-\frac{1}{2}} & \textrm{if $d=1$,}\\
    n^{-\frac{1}{2}}\sqrt{\log (1 + n)} & \textrm{if $d=2$,}\\
    n^{-\frac{1}{d}} & \textrm{if $d \geq 3$.}
    \end{array} 
    \right.
    \end{aligned}
\end{equation*}
Furthermore, for any $n \geq 1, t \in (0,\infty)$ and for each $r \in (0, 3)$
\begin{align*}
    &\mathbb{P}\left( \Delta(f^n;\mu) \geq t \right)\\
    &\leq C
    n (nt)^{-(3 -r)} + C \mathds{1}_{\{t\leq 1\}}
    \left\{\begin{array}{ll}
    \exp(-cnt^2) & \hbox{if $d=1$}, \\[+3pt]
    \exp(-cn(t/\log(2+1/t))^2) & \hbox{if $d=2$}, \\[+3pt]
    \exp(-cn t^{d}) & \hbox{if $d \geq 3$}.
\end{array}\right.
\end{align*}
Here, the positive constants $C$ and $c$ depend only on $d$, $M_3(\mu)$ and $r$.

    \item[(ii)] Assume that $\mu$ has a bounded support with diameter $D$, and $k > d^*_1(\mu) \vee 2$ where $d^*_p(\mu)$ is the upper $p$-Wasserstein dimension for $\mu$. Then there exist constants $C=C(k) > 0$ such that
\begin{align*}
    \mathbb{E} \Delta(f^n;\mu) \leq 2 (2 \vee L_c) D^2  \left(3^{\frac{3k}{k - 2} + 1}\left( \frac{1}{3^{\frac{k}{2} - 1} - 1} + 3 \right) n^{- \frac{1}{k}} + C^{\frac{k}{2}} n^{- \frac{1}{2}} \right). 
\end{align*}
Furthermore, for any $n \geq 1, t \in (0,\infty)$,
\begin{align*}
    &\mathbb{P} \left(  \Delta(f^n;\mu) \geq \mathbb{E}  \Delta(f^n;\mu) + t \right) \leq \exp\left(-2n\frac{t^2}{4D^2} \right).
\end{align*}
\end{enumerate}    
\end{corollary}

\section{Conclusion and future work}
In this work, we establish central limit theorems (CLTs) for adversarial training. We prove two forms of CLT: (i) one for the empirical adversarial risk centered by its expectation, and (ii) another for the smoothed adversarial risk centered by its population counterpart. As by-products, we also establish the stability of robust classifiers and optimal adversarial attacks and sample-complexity bounds for the generalization error. These results are of independent interest to the machine-learning community.

Our analysis draws on several strands of work developed in the machine-learning and optimal-transport communities. The first is the recently established connection between optimal transport and adversarial training. This connection yields dual formulations of the adversarial learning problem that play a central role in our convergence and uniqueness analyses.

The second one is the lifting of the collection $(g_i)_{i \in \mathcal{Y}}$ to the vector-valued function $\bm{g}$, which enables us to use the Efron-Stein argument. Thanks to the natural boundedness, one can derive the first CLT under the hypothesis of a unique optimal potential. It is crucial that the uniqueness up to additive constants of the optimal potential is not sufficient for the CLT due to the existence of non-vanishing covariance terms, which arise from the dependence structure
induced by the mixture structure of $\mu$. To the best of the authors' knowledge, these are not shown in the previous statistical optimal transport literature.

The third ingredient is empirical process theory. Using empirical-process techniques, we establish the tightness of the Gaussian process indexed by the function class induced by the smoothed adversarial learning problem. Using the boundedness of the potential, we can show the Donsker property of the lifted function space, which will be of interest to the learning community. Combining this result with the stability of the optimal potential achieved by the new regularity result, we prove the second CLT for the smoothed adversarial training risk.

The last one is regarding the optimal partial transport formulation, which is crucial to achieve the genuine uniqueness of the optimal potential. In the binary setting, the adversarial training problem can be formulated exactly as an optimal partial transport problem. This reformulation allows us to connect optimal partial transport to standard optimal transport and thereby establish uniqueness of the optimal potential for the non-degenerate case.

We conclude by outlining two directions for future research.

The first direction is to extend our results to neural-network settings. Our analysis is conducted within the agnostic learning framework, in which classifiers may be arbitrary measurable functions. Although this assumption is theoretically convenient, it does not reflect the structural constraints imposed by modern machine-learning architectures. Given the central role of adversarial examples in deep learning, understanding the statistical properties of adversarial training over neural-network classes is important from both theoretical and practical perspectives.

The second direction is to study more general loss functions. Although the $0$--$1$ loss is valuable for theoretical analysis, it is rarely used in practice. Practitioners instead typically employ convex surrogate losses because they are more amenable to optimization. Adversarial learning models with nonlinear loss functions have recently been studied in \cite{trillos2025lowerboundsadversarialrobustness}. It would be worthwhile to investigate the statistical properties of such models, including their generalization behavior and asymptotic distributions.

\bibliographystyle{plainnat}
\bibliography{reference.bib}

\begin{appendices}
\section{Rate of the Wasserstein distance of the empirical measures}
The rate of the expectation of the Wasserstein distance of the empirical measures and its concentration inequalities are developed in \cite[Theorems 1]{NF_AG_rate_Wasserstein} (with finite higher moments) and \cite[Proposition 5]{JW_FB_sample_rates} (with bounded support).

\begin{lemma}[Upper bound of the expected empirical Wasserstein distance]\cite[Theorems 1]{NF_AG_rate_Wasserstein}\cite[Proposition 5]{JW_FB_sample_rates} \label{thm : upper bound of expected distance}
\begin{enumerate}
    \item[(i)] Let $\mu$ be a probability measure on $\mathbb{R}^{d}$ with a finite $3$-th moment, and let $\mu^{n}$ be an empirical measure for $\mu$. Then, for all $n \geq 1$,
\begin{equation}\label{eq: convergence rate of FG15}
    \begin{aligned}
    \mathbb{E}[\mathcal{W}_1(\mu^n,\mu)] \leq 20 d M_3(\mu) 
    \left\{ 
    \begin{array}{ll}
    n^{-\frac{1}{2}} & \textrm{if $d=1$,}\\
    n^{-\frac{1}{2}}\sqrt{\log (1 + n)} & \textrm{if $d=2$,}\\
    n^{-\frac{1}{d}} & \textrm{if $d \geq 3$.}
    \end{array} 
    \right.
    \end{aligned}
\end{equation}
    \item[(ii)] Let $\mu$ be a probability measure on $\mathbb{R}^{d}$ with a bounded support with diameter $D$. If $k > d^*_1(\mu) \vee 2$, then there exists constant $C=C(k) > 0$ such that
\begin{equation}\label{eq: convergence rate of WB19}
    \mathbb{E}[ \mathcal{W}_1(\mu^n,\mu)] \leq D^2  3^{\frac{3k}{k - 2} + 1}\left( \frac{1}{3^{\frac{k}{2} - 1} - 1} + 3 \right) n^{- \frac{1}{k}} + D^2 C^{\frac{k}{2}} n^{- \frac{1}{2}}  
\end{equation}
where $d^*_p(\mu)$ is the upper p-Wasserstein dimension.
\end{enumerate}
\end{lemma}

\begin{remark}
$d^*_p(\mu)$, the upper $p$-Wasserstein dimension is introduced by \citet{JW_FB_sample_rates}. This is a fractal dimension like Hausdorff and Minkowski ones.

In \eqref{eq: convergence rate of WB19} the constant of $n^{-\frac{1}{k}}$ decreases as $k$ increases as long as $k > 2$ while the constant $C$ of the second term depends on $k$ exponentially. If $d^*_1(\mu) \leq 2$, then one can make a choice of $k > 2$ freely to minimize the RHS of \eqref{eq: convergence rate of WB19}.

If $\mu$ is supported on a regular set of Hausdorff dimension $d$ and $\mu \ll \mathcal{H}^d$, the d-dimensional Hausdorff measure, then $d^*_p(\mu) = d$ for any $p \in [1, \infty)$. For example, if $\mu$ is absolutely continuous with respect to Lebesgue measure, then $d^*_p(\mu) = d$: see \cite[Propositions 8 and 9]{JW_FB_sample_rates} for more details.    
\end{remark}

\begin{lemma}[Concentration inequalities]\cite[Theorems 2]{NF_AG_rate_Wasserstein}\cite[Propositions 20]{JW_FB_sample_rates}\label{lem:concentration}
\begin{enumerate}
    \item[(i)] Let $\mu \in \mathcal{P}_3(\mathbb{R}^d)$. Then for all $n \geq 1$, all $t \in (0,\infty)$,
\[
    \mathbb{P}(W_1(\mu^n,\mu) \geq t) \leq a(n,t)\mathds{1}_{\{t\leq 1\}}+b(n,t, r),
\]
where
\[
    a(n,t)=C 
    \left\{\begin{array}{ll}
    \exp(-cnt^2) & \hbox{if $d=1$}, \\[+3pt]
    \exp(-cn(t/\log(2+1/t))^2) & \hbox{if $d=2$}, \\[+3pt]
    \exp(-cn t^{d}) & \hbox{if $d \geq 3$}
\end{array}\right.
\]
and for each $r \in (0, 3)$,
\[
    b(n,t, r)= C
    n (nt)^{-(3 -r)}.
\]
Here, the positive constants $C$ and $c$ depend only on $d$, $M_3(\mu)$ and $r$.
    \item[(ii)] Let $\mu$ be a probability measure on $\mathbb{R}^{d}$ with a bounded support with diameter $D$. For any $1 \leq p < \infty$,
\begin{equation*}
    \mathbb{P}\left( W_p^p(\mu^n, \mu) \geq \mathbb{E} W_p^p(\mu^n, \mu) + t \right) \leq \exp\left(-2n\frac{t^2}{D^2}\right).
\end{equation*}
\end{enumerate}
\end{lemma}
   
\end{appendices}

\end{document}